\documentclass{article}

\PassOptionsToPackage{numbers, compress}{natbib}
\usepackage[preprint]{neurips_2026}

\usepackage[utf8]{inputenc} 
\usepackage[T1]{fontenc}    
\usepackage{hyperref}       
\usepackage{url}            
\usepackage{booktabs}       
\usepackage{amsfonts}       
\usepackage{nicefrac}       
\usepackage{microtype}      
\usepackage[most]{tcolorbox}
\usepackage{xcolor}         
\usepackage{amsmath}
\usepackage{amssymb}
\usepackage{amsthm}
\usepackage{graphicx}
\usepackage{wrapfig}
\usepackage{subcaption}
\newcommand{\E}{\mathbb{E}}

\DeclareMathOperator*{\argmin}{arg\,min}
\theoremstyle{plain}
\newtheorem{theorem}{Theorem}[section]
\newtheorem{proposition}[theorem]{Proposition}
\newtheorem{lemma}[theorem]{Lemma}

\newtheorem{claim}[theorem]{Claim}
\theoremstyle{definition}
\newtheorem{definition}[theorem]{Definition}
\newtheorem{assumption}[theorem]{Assumption}
\theoremstyle{remark}

\definecolor{lightgrey}{RGB}{245, 245, 245}
\definecolor{darkgrey}{RGB}{120, 120, 120}

\title{Performative Prediction with Selective Labels}

\author{%
  Giovani Valdrighi\thanks{Correspondence to: \texttt{giovani.valdrighi@ic.unicamp.br}} \\
  Instituto de Computação \\
  Universidade Estadual de Campinas \\
  Campinas, Brazil \\
  \And
  Isabel Valera \\
  Department of Computer Science \\
  Saarland University \\
  Saarbrücken, Germany \\
  \And
  Marcos Medeiros Raimundo \\
  Instituto de Computação \\
  Universidade Estadual de Campinas \\
  Campinas, Brazil \\
}

\begin{document}

\maketitle

\begin{abstract}
Many social applications of machine learning exhibit performative effects: population behavior changes in response to deployed models. Performative prediction studies this interaction through a distribution map that relates each model to the population distribution it induces. One of the main results in this framework showed that repeated risk minimization (RRM), which updates models by retraining on the most recent data, can converge to a stable model that minimizes risk on its own induced distribution. However, existing analyses typically assume access to the complete distributions of features and labels after model deployment, ignoring the possibility of \textit{selective labels}: observing labels only for the accepted subset of the population. In this work, we formalize performative prediction with selective labels and show that retraining only on observed data can misguide the retraining procedure and undermine the guarantees of convergence to a stable solution. We then propose a worst-case objective based on knowledge of a confidence interval on the probability of a positive label. Applying RRM to this objective permits us to remain within a bounded distance to the true stable point. Under a sensitivity assumption on the conditional label distribution, we further show how previously accepted data can tighten these confidence intervals over time. Experiments in a lending application with fairness regularization show that our robust optimization approach closely matches the performance of RRM with complete label access.

\end{abstract}

\section{Introduction}
\label{sec:introduction}

In algorithmic decision-making, machine learning models support decisions and intervene on the outcomes they were trained to predict. For example, a high-risk prediction for a certain disease can prompt early treatment and help avoid this outcome. Additionally, highly consequential applications such as hiring, finance, and criminal justice define opportunities and incentives, thereby shaping the populations they interact with. The seminal work by \citet{perdomo2020performative} formalized these ideas in the framework of \textit{performative prediction}, where each model $\theta$ induces a distribution $(X, Y) \sim D(\theta)$. In this sense, a model optimal under loss $\ell$ for a training distribution $D$ is not necessarily optimal when deployed under the distribution $D(\theta)$. This shift can be expressed by the decoupled risk, which separates the model $\phi$ that induces the distribution and the model $\theta$ being evaluated:
\begin{align}
    J(\phi, \theta) = \E_{D(\phi)}[\ell(X, Y, \theta)]
\end{align}
One of the main results from \citet{perdomo2020performative} was to show that if $\ell$ is strongly convex and smooth and $D(\theta)$ has bounded sensitivity w.r.t. $\theta$, the iterative procedure of retraining the model in the most recent distribution would converge to an equilibrium where $\theta$ was optimal in the induced distribution $D(\theta)$, that is, $\theta^{PS} = \argmin_\theta J(\theta^{PS}, \theta)$. This means $\theta^{PS}$ is robust to the distribution shift it induces. However, despite the practical relevance of performative prediction, many works still assume access to samples from the full distribution $D(\theta)$, ignoring that many decision-making applications have \textbf{selective labels}: labels are observed only for the accepted subset of the population. In a lending application, the decision-maker can access the economic profile and financial history from all applicants, while only observing the payment label for the subset that had their loan accepted; in a health diagnosis setting, patients' profiles are accessible, while the final diagnosis is only known for the subset that underwent testing.

Performative prediction is a widely discussed topic, with multiple works examining conditions for achieving performative stability, such as stochastic learning \citep{mendler2020stochastic, li2024stochastic, li2025clipped}, introducing dependence on previous states \citep{brown2022performative}, or relaxing assumptions on the loss function \citep{mofakhami2023performative}. In general, this literature considers variations of an iterative learning process called repeated risk minimization (RRM) in which, after selecting a model $\theta_t$, it is deployed in the world to obtain access to $D(\theta_t)$, which is then used to define the next model $\theta_{t+1}$. However, many relevant applications suffer from selective labels: while features $x$ are collected from all samples, the label $y$ is observed only for a subset of samples accepted under the current $\theta$. In that sense, we never observe the true distribution $D(\theta)$, but rather the conditional distribution $X, Y | O=1$, where $O=1$ is the acceptance event.

Our contributions are:
\begin{itemize}
    \item \textit{The Selective Performative Gap}: We demonstrate that retraining with performativity fails to reach stability under selective labels. When observation is deterministic and threshold-based, models trained on the fully observed distribution $D(\theta)$ versus the filtered distribution $\tilde{D}(\theta)$ diverge arbitrarily. This fundamental gap motivates our new framework.
    \item \textit{Robust Repeated Risk Minimization}: To overcome missing labels, we introduce a strongly convex upper-bound loss that operates independently of the true label $y$, extending partial identification techniques \citep{chen2025learning}. By leveraging limited knowledge of $P(Y=1|X=x, \theta)$, we bound the decoupled performative risk and optimize the model $\theta_t$ through our novel Robust RRM procedure.
    \item \textit{Theoretical Guarantees}: We prove that Robust RRM approximates the true performatively stable solution $\theta^{PS}$ within a bounded distance. This bound tightens with higher label accessibility, better distributional knowledge, and lower gradient sensitivity. Furthermore, we provide a principled method for iteratively tightening these bounds during retraining, based on prior observations.
    \item \textit{Novel Evaluation Framework}: We propose a new experimental setting for performative prediction with selective labels, adapting the lending environment with fairness constraints introduced by \citet{liu2018delayed} to validate our robust retraining dynamics.
\end{itemize}

Code for experiments is available at \href{https://github.com/recod-ai/perf_selective}{\nolinkurl{github.com/recod-ai/perf_selective}}.

\subsection{Related Work}

For a comprehensive review of prior work in performative prediction, we refer to the surveys by \citet{hardt2025performative} and \citet{kehrenberg2026dissecting}.

\paragraph{Performative Prediction}
Motivated by ideas of performativity in economics \citep{mackenzie2007economists}, \citet{perdomo2020performative} introduced performative prediction in supervised learning and studied conditions for stable, optimal solutions. Since then, a larger literature has arisen in performative prediction, with works studying how to reach performative optimality with structural assumptions over $D(\theta)$ \citep{izzo2021learn, izzo2022learn, miller2021outside} or bandit-inspired algorithms \citep{liu2024two, jagadeesan2022regret}, considering constraints \citep{yan2023zero}, fairness \citep{mishler2022fair, hu2022achieving, jin2026addressing}, or adapting it to reinforcement learning \citep{mandal2023performative}. Reaching a stable solution has also been a relevant topic, with works studying the convergence to the stable solution of stochastic algorithms \citep{mendler2020stochastic, li2022state, li2024stochastic, li2025clipped}, weakening the assumption of strong-convexity of the loss \citep{mofakhami2023performative}, introducing state dependency \citep{brown2022performative}, reusing data from previous deployments to improve convergence rates \citep{khorsandi2025tight}, or reaching a fair stable point \citep{jin2026addressing}. 

\paragraph{Selective Labels}
Selective labels is a special case of selection bias, a larger theme that studies settings where not all observations are available to the decision-maker. This problem has already been studied in prior work using static data distributions \citep{kilbertus2020fair, lakkaraju2017selective, chen2025learning} and its interaction with fairness \citep{rateike2022don, keswani2024fair, chang2024biased, frauen2024fair}. As \citet{fawkes2024the} point out, most fairness benchmarks exhibit some form of selection bias. Yet, few works have considered the interaction between dynamic distributions and selective labels. \citet{creager_causal_2020} and \citet{valdrighi2026longterm} both discuss fairness in dynamic environments with selective labels, with the former leveraging a doubly robust estimator and the latter presenting a reinforcement-learning algorithm. Yet neither solution can handle decisions made using thresholds.

A few works have already discussed particular cases of performative prediction with selective labels. In strategic classification \citep{hardt2016strategic}, where individuals can alter their features to obtain the desired prediction without truly altering their label, \citet{harris2023strategic} present online learning algorithms with bounded regret for classification using linear policies. \citet{zhao2025online} extend these results by accounting for noise in label observations. Despite their similarity to this work, both focus only on the strategic classification scenario. \citet{pombal2022prisoners} discuss how selective labels can introduce unfairness in performative prediction settings and presents a fraud-detection simulation.

\section{Preliminaries}
\label{sec:preliminaries}

We consider a population described by features $x \in \mathcal X \subseteq  \mathbb R^d$ and a binary label $y \in \mathcal Y :=\{0,1\} $. The decision-maker acts in the system with a model defined by parameters $\theta$ from a compact convex set $\Theta \subset \mathbb R^d$, which incurs loss defined by $\ell : \mathcal X \times \mathcal Y \times \Theta \to \mathbb R$. In a loan example, $x$ could be socioeconomic indicators, $y$ the loan payment, and $\ell$ is the binary cross-entropy. We will write the Euclidean norm as $\|\cdot\|_2$, $\mathbf{1}\{\cdot\}$ as the indicator function and use $\theta$ and $\phi$ to refer to arbitrary models from $\Theta$.

\paragraph{Performative Prediction} After deploying a model $\theta$, the population responds with a new distribution defined by the distribution map $D : \Theta \to \Delta(\mathcal X \times \mathcal Y)$. Thus, a training dataset $D$ will not reflect the true performance of a model $\theta$ when deployed. We will write $p_\theta(x)$ as the probability density of $x$ with the distribution $D(\theta)$ and define the label distribution $\alpha_{\theta}(x) = P(Y = 1 | X = x, \theta)$. Many applications of social decision-making will present performativity. In a loan application, individuals might try to manipulate their features $x$ to reach the desired approval, without necessarily altering the true label $y$. In this example, both $p_{\theta}(x)$ and $\alpha_{\theta}(x)$ will be dependent on $\theta$. Another application commonly discussed in long-term fairness studies \citep{liu2018delayed, rateike2024designing, zezulka2024fair} assumes that individuals will improve their condition (or have their condition altered by $\theta$), resulting in an update to $x$, which is then reflected in $y$. In this scenario, $p_{\theta}(x)$ remains $\theta$-dependent, while $\alpha_{\theta}$ becomes invariant to $\theta$. In the performative setting, the objective of the decision-maker is to minimize the performative risk:
\begin{align}
\min_{\theta \in \Theta} \E_{D(\theta)}[\ell(X, Y, \theta)]
\end{align}
However, this objective depends on $\theta$ both in the loss function and in the distribution, which we observe only after model deployment. We also define the decoupled performative risk $J(\phi, \theta) := \E_{D(\phi)}[\ell(X, Y, \theta)] $, where a model $\theta$ is evaluated in the distribution induced by another model $\phi$.

\paragraph{Stable Solutions}
In the performative prediction setting, two types of solutions are of interest. A performatively optimal solution is $\theta^{PO} = \argmin_{\theta \in \Theta} J(\theta, \theta)$, and a performatively stable solution is such that $\theta^{PS} =  \argmin_{\theta \in \Theta} J(\theta^{PS}, \theta)$. Performatively stable solutions can be interpreted as an equilibrium point, where $\theta^{PS}$ minimizes the loss over the distribution it induces, and the decision-maker has no incentive to select another model. Under the conditions we present below, retraining on the most recent data distribution converges to a performatively stable solution.

\begin{definition}[RRM]
\label{def:rrm}
Repeated Risk Minimization (RRM) is the procedure of starting from a model $\theta_0$, employing the iterative updates based on $\theta_{t+1} = G(\theta_t) := \argmin_{\theta \in \Theta} J(\theta_t, \theta)$.
\end{definition}

\begin{definition}[Sensitivity and regularity conditions \citep{perdomo2020performative}]
A distribution map $D : \Theta \to \Delta(\mathcal X \times \mathcal Y)$ is $\epsilon$-sensitive if $\forall \theta, \phi \in \Theta$, it holds that $\mathcal W_{1}(D(\theta), D(\phi)) \le \epsilon \|\theta - \phi\|_2$ where $\mathcal W_1$ is the Wasserstein-1 distance. Additionally, the loss $\ell : \mathcal X \times \mathcal Y \times \Theta \to \mathbb R$ is:

\begin{itemize}
    \item $\gamma$-strongly convex if $\forall \theta, \phi \in \Theta$ and $(x, y) \in \mathcal X \times \mathcal Y$, $\ell(x, y, \theta) \ge \ell(x, y, \phi) + \nabla_\theta \ell(x, y, \phi)^\top(\theta- \phi)+\frac{\gamma}{2}\|\theta- \phi\|_2^2$.
    \item $\beta$-smooth in $(x, y)$ if $\|\nabla_\theta \ell(x, y, \theta) - \nabla_\theta \ell(x', y', \theta)\|_2 \le \beta\|(x,y) -(x', y')\|_2, \forall \theta \in \Theta, (x, y), (x', y') \in \mathcal X \times \mathcal Y$.
\end{itemize}
\end{definition}

\begin{theorem}[Theo. 3.5 from \citep{perdomo2020performative}]
\label{theo:performative}
If the loss $\ell$ is $\beta$-smooth in $(x,y)$ and $\gamma$-strongly convex, the distribution map is $\epsilon$-sensitive, then RRM as in Def.~\ref{def:rrm} is a contraction satisfying:
\begin{align}
\|G(\theta) - G(\phi)\|_2 \le \dfrac{\epsilon \beta}{\gamma} \|\theta - \phi \|_2    
\end{align}
and if ${\epsilon \beta}/{\gamma} < 1$, the iterates $\theta_t := G(\theta_{t-1})$ will converge to a performatively stable solution $\theta^{PS}$.   
\end{theorem}

However, as we argue in the following section, the decision-maker will commonly not have access to samples of the full distribution $D(\theta)$. 

\section{Selective Labels in Performative Prediction}

\begin{figure}[ht]
\centering
\begin{tcolorbox}[
    width=0.9\textwidth,
    colback=lightgrey,        
    colframe=darkgrey,        
    arc=3mm,                  
    boxrule=1pt,            
    boxsep=1em,               
    left=0pt, right=0pt, top=0pt, bottom=0pt 
]
\noindent \textbf{for} $t = 1, 2, \dots$ \\[0.3em]
\hspace*{1.5em} 1. Decision-maker publishes a model $\theta_t$ \\[0.3em]
\hspace*{1.5em} 2. Environment (the population) realizes feature-label pair $(X, Y) \sim D(\theta_t)$ \\[0.3em]
\hspace*{1.5em} 3. Decision-maker observes $X$ \\[0.3em]
\hspace*{1.5em} 4. Decision-maker computes score $f_{\theta_t}(X)$ and defines $O := \mathbf{1}\{f_{\theta_t}(X) \geq \tau\}$ \\[0.3em]
\hspace*{1.5em} 5. If $O = 1$, decision-maker observes $Y$
\end{tcolorbox}
\caption{Observation dynamics for performative prediction with selective labels.}
\label{fig:performative_selective}
\end{figure}

In real-world applications such as loan requests, school and job admissions, healthcare, and criminal justice, decision-makers observe $x$ but have access to labels $y$ only for a subset of the population that was accepted, called \textit{selective labels} \citep{lakkaraju2017selective}. For example, in judicial bail decisions, judges have access to criminal history $x$, while the label $y$ representing recidivism is observed only for defendants released pretrial \citep{lakkaraju2017selective}. In health diagnosis, patient covariates, such as routinely collected clinical measurements, may be available before a definitive diagnostic test is ordered. A label defined by the final diagnosis is therefore available only for the subset of patients who underwent testing \citep{chang2024biased}. In both settings, deployed policies may also alter subsequent actions and outcomes, thereby inducing performative effects \citep{boeken2024evaluating}.

We extend the performative prediction framework to include selective labels as displayed in Fig.~\ref{fig:performative_selective}. In this extension, $\theta$ also determines label observations. More specifically, after deploying $\theta$, the population/environment observes it and reacts with the pair $X, Y$ sampled from $D(\theta)$. The decision-maker first observes $X$ and, using the model $\theta$, computes a score via $f_\theta: \mathcal X \to \mathbb R$. In this work, we focus on the deterministic, threshold-based observation case: $O = \mathbf{1}\{ f_\theta(X) \ge \tau\}$, where $\tau \in \mathbb R$ is a threshold. Finally, $Y$ is revealed only upon acceptance, that is, when $O=1$. Prior online learning work by \citet{kilbertus2020fair} shows that such a configuration is challenging and can yield suboptimal performance even when the true distribution $D$ is fixed.

To employ RRM, after deploying $\phi$, the decision-maker will leverage the data to evaluate the loss of another model $\theta$, that is, calculating $\mathbb E_{D(\phi)}[ \ell(X, Y, \theta)]$. Yet, $\ell$ cannot be calculated for rejected samples. A naive solution for the selective labels problem would be to leverage only the data from the accepted population, that is, the distribution $\tilde D(\phi)$ that results from first sampling $(X, Y) \sim D(\phi)$ and then keeping observations if $f_{\phi}(X) \ge \tau$. However, as our next example shows, this procedure can produce models arbitrarily far from the true performatively stable solution $\theta^{PS}$.

\begin{proposition}
Define $\tilde G(\phi):= \argmin_{\theta \in \Theta} \E_{\tilde D(\phi)}[\ell(X, Y, \theta)]$ and let $G$ be the RRM as Def.~\ref{def:rrm}. The gap of one iteration with a model $\theta$, measured by the distance $\|G(\theta) - \tilde G(\theta)\|_2$ can reach the maximum value of $\mathrm{diam}(\Theta) :=\sup_{\theta', \theta'' \in \Theta} \|\theta' - \theta''\|_2$.
\begin{proof}
We will prove this with a simple example where $\|G(\theta) - \tilde G(\theta)\|_2 = \mathrm{diam}(\Theta)$ even when $D$ is not dependent on $\theta$. Consider a parameter space restricted to $\Theta = [0, B]$ for a fixed $B \in \mathbb R_{>0}$. Let $x \in \{-1, 1\}$ and the loss is the binary cross entropy $\ell(x, y, \theta) = -y \log(\sigma(\theta x)) - (1 - y) \log(1 - \sigma(\theta x))$. The decision-maker uses a linear model $f_\theta(x) = \theta x$ and a threshold $\tau = 0$. Define a fixed distribution such that $P(X = -1 \land Y = 1) = 0.9$ and $P(X = 1 \land Y = 1) = 0.1$. The loss is strongly convex on $\Theta$, and for any fixed $\phi$, $J(\phi, \theta) = 0.9 \log(1 + e^{\theta}) + 0.1 \log(1 + e^{-\theta})$ has derivative $\sigma(\theta) - 0.1 > 0$ on $\Theta$, so it is strictly increasing in $\theta$ and its minimizer is $0$. Consequently, $G(\phi) = 0$ for every $\phi$. However, after deploying any model $\phi \in (0, B]$, only samples with $x = 1$ satisfy $f_\phi(x) = \phi x \ge 0$, so we observe only samples with $x=1$ and $y=1$. On this selectively observed distribution, the loss is strictly decreasing and is minimized when $\theta = B$. Therefore, $|G(\phi) - \tilde G(\phi) | = B = \mathrm{diam}(\Theta)$.  
\end{proof}
\end{proposition}

This proposition shows that retraining on $\tilde D(\cdot)$ can end arbitrarily far from the performatively stable solution $\theta^{PS}$ obtained from RRM, even when no performative effects exist. In this way, training on the selectively labeled data can be an inappropriate objective depending on the true distribution map $D(\cdot)$. Next, we study how to leverage the features $x$ of rejected samples to ensure proximity to the true decoupled risk.

\section{A Robust Optimization Approach}

As our previous result showed, using only the accepted population to perform RRM can cause a departure from the stable point, so we present a robust optimization approach. Inspired by prior work on partial identification from \citet{chen2025learning}, the proposed approach leverages a worst-case estimate of the decoupled risk. This lets us show that running RRM on this worst-case loss approximates the stable solution $\theta^{PS}$ within a bounded distance. We present all proofs in Appendix~\ref{app:proofs}. We first assume partial knowledge of the label distribution $\alpha_{\theta}$.

\begin{assumption}
\label{as:limits}
There are functions $l : \mathcal X \times \Theta \to [0, 1], u : \mathcal X \times \Theta \to [0, 1]$ such that for all $\theta \in \Theta, x \in \mathcal X$, $l(x, \theta) \le \alpha_\theta(x) \le u(x, \theta)$.
\end{assumption}

This assumption is always valid in the trivial case where $l(x, \theta) = 0, u(x, \theta) = 1, \forall x, \theta$. Yet, as the following results show, improving these intervals is an important step toward tightening our robust approach. Later, we present a specific formulation that leverages additional knowledge about the label distribution. With intervals $[l(x, \theta), u(x, \theta)]$, we can define a worst-case loss as follows:

\begin{definition}[Worst-Case Loss]
\label{def:wc_loss}
With a loss $\ell : \mathcal X \times \mathcal Y \times \Theta \to \mathbb R$, define $\Delta_\ell(x, \theta) := \ell(x, 1, \theta) - \ell(x, 0, \theta)$ as the difference between the loss evaluated on $y=1$ and $y=0$ for sample $x$ and model $\theta$. The worst-case loss\footnote{We discuss extending the worst-case loss for the multi-class setting in Appendix~\ref{app:multi_class}.} for sample $x$ under the distribution $D(\phi)$ is:
\begin{align}
     \overline \ell_\phi(x, \theta) &:=  \ell(x, 0, \theta) + \log \left (e^{l(x, \phi ) \Delta_\ell(x, \theta )}+e^{u(x, \phi )\Delta_\ell(x, \theta)} \right )
 \end{align}
 \end{definition}
The worst-case loss upper-bounds the conditional expected loss and remains differentiable and strongly convex.
\begin{claim}
\label{claim:bound}
If AS.~\ref{as:limits} holds, for any distribution $D(\phi)$ over $X, Y$, it holds that $\mathbb E_{D(\phi)}[\ell(X, Y, \theta)]  \le \mathbb E_{D(\phi)}[\overline \ell_\phi(X, \theta)]$. Additionally, if $\ell$ is $\gamma$-strongly convex and $\beta$-smooth, then $\overline \ell_\phi$ is $\gamma$-strongly convex but not necessarily $\beta$-smooth.
\end{claim}
Finally, we can leverage our worst-case loss $\overline \ell_\phi$ for samples for which we do not have access to true labels, while using the loss $\ell$ for accepted samples. By using the true loss measure on accepted samples, we avoid losing information from the observed labels.
\begin{definition}[Worst-case Decoupled Risk]
The worst-case decoupled risk for a loss $\ell$, with the defined respective worst-case loss as in Def.~\ref{def:wc_loss}, employs $\ell$ for samples with label-access and $\overline \ell_\phi$ for samples without access, as:
\begin{align}
    \overline J(\phi, \theta) := \mathbb E_{D(\phi)}[O \ell(X, Y, \theta) + (1 - O)\overline \ell_\phi(X,  \theta)]
\end{align}
\end{definition}

\begin{definition}[Robust RRM]
\label{def:robust_rrm}
Robust Repeated Risk Minimization is the procedure of starting from a model $\theta_0$, employing the iterative updates based on $\theta_{t+1} = \overline G(\theta_t) := \argmin_{\theta \in \Theta} \overline{J}(\theta_t, \theta)$.
\end{definition}

Since $\overline \ell_\phi$ is not necessarily $\beta$-smooth, we cannot directly apply Theo.~\ref{theo:performative} by replacing $\ell$ with $\overline \ell_\phi$. In that case, convergence to a stable solution is not guaranteed, even to a solution different from $\theta^{PS}$. In the remainder of this section, we show that Robust RRM remains within a bounded distance of the stable point $\theta^{PS}$ obtained with access to all labels. Our next result studies this by comparing one-step updates under both risks.

\begin{theorem}[One-step gap between full and robust updates]
\label{theo:model_gap}
Fix a deployment model $\phi$. Suppose $\ell$ is differentiable and $\gamma$-strongly convex on $\theta$ and AS.~\ref{as:limits} holds. Let $G$ be the RRM as in Def.~\ref{def:rrm} and $\overline G$ be the Robust RRM as in Def.~\ref{def:robust_rrm}.
Then
\begin{align}
\|G(\phi)-\overline G(\phi)\|_2
\le
\frac{1}{\gamma}
\sup_{\theta\in\Theta}
\mathbb E_{D(\phi)}
\left[
(1-O)
|u(X,\phi)-l(X,\phi)|
\|\nabla_\theta\Delta_\ell(X,\theta)\|_2
\right].
\end{align}
In particular, if $\ell$ is the binary cross-entropy of a linear score $f_\theta(x)=\theta^\top x$ and $\forall x \in \mathcal X, \|x\|_2 \le R$:
\begin{align}
\|G(\phi)-\overline G(\phi)\|_2
\le
\frac{R}{\gamma}
\mathbb E_{D(\phi)}
\left[
(1-O)
|u(X,\phi)-l(X,\phi)|
\right].
\end{align}
\end{theorem}
This theorem shows that one iteration performed both in $ G$ and $ \overline G$ will result in two models with a limited distance, which is related to the gap between gradients of the loss function on $y=0$ and $y=1$ and the width of the limits $l$ and $u$ over the induced distribution of $D(\phi)$. Since robust updates remain within a bounded distance of the true updates, and true updates converge to a stable solution, we can bound how close iteration with $\overline G$ stays to that stable solution.

\begin{theorem}[Approximate stability of Robust RRM]
\label{theo:convergence}
Assume the conditions of Theo.~\ref{theo:performative}, such that the full update $G$ is an $\eta$-contraction with $\eta = \epsilon \beta / \gamma < 1$, and let $\theta^{PS}$ denote its unique fixed point. Suppose AS.~\ref{as:limits} holds and let $\kappa := \sup_{\theta \in \Theta} \|\overline G(\theta)-G(\theta)\|_2$. 
For the robust iterates $\theta_{t+1}=\overline G(\theta_t)$:
\begin{align}
\|\theta_t-\theta^{PS}\|_2
\le
\eta^t\|\theta_0-\theta^{PS}\|_2
+
\frac{1-\eta^t}{1-\eta}\kappa
\end{align}
Consequently, $\limsup_{t\to\infty}\|\theta_t-\theta^{PS}\|_2 \le \frac{\kappa}{1-\eta}$.
\end{theorem}
While this result does not prove the convergence of Robust RRM, it shows that it remains within a bounded distance of the stable solution with total access to labels. From Theo.~\ref{theo:model_gap}, this distance $\kappa$ relates to the strong-convexity constant $\gamma$ of the loss $\ell$, the  width of $|u(\cdot, \cdot)-l(\cdot, \cdot)|$ and the magnitude of the gradient $\|\nabla_\theta\Delta_\ell(X,\theta)\|_2$. In Appendix~\ref{app:bound_performative}, we also present an upper bound on the performative risk of Robust RRM with some additional conditions on $\ell$. As $\ell$ is application-dependent and cannot be easily altered, we study how to reduce the width of intervals during learning.

\subsection{Improving Label Distribution Limits}
\label{sec:improving}

In this section, we present an approach to estimate limits $l$ and $u$ for data with unseen labels based on the knowledge of $\alpha_{\theta}$. We do this by introducing sensitivity constants that are specific to the label distribution as follows:

\begin{definition}
The label distribution $\alpha_\theta(x)$ is $(L_x, L_\theta)$-sensitive if for all $x, x' \in \mathcal X$ and $\theta, \phi \in \Theta$, $|\alpha_\theta(x) - \alpha_{\phi}(x')| \le L_x \|x-x'\|_2 + L_\theta \| \theta - \phi\|_2$.
\end{definition}

This definition controls how $\alpha_\theta(x)$ can change by varying $x$ and $\theta$. In this way, we can construct a limit for unseen $\alpha_\theta(x)$ based on the knowledge of observed samples $x$ and models $\theta$. In particular, if a label distribution $\alpha_\theta$ is invariant to $\theta$, we can conclude that it is $(L_x, 0)$-sensitive. In the remainder of this section, we will focus on the case where $\alpha_\theta$ is invariant to $\theta$, and we will simplify the notation to $\alpha(x), l(x), u(x)$. See Appendix~\ref{app:results_dependent} for the results when $L_\theta \neq 0$.

During the repeated risk minimization procedure, the decision-maker has access to data from previously accepted individuals, that is $X, Y | O = 1$, for all previously deployed models $\theta_1, \dots, \theta_t$. At time $ t$, this information can be used to assess the limit $\alpha$ for samples rejected by $\theta_t$. Let $\mathcal X^{\mathrm{obs}}_{1:t}$ be the support of the features accepted by $\theta_1, \dots, \theta_t$. Then, for an unobserved sample $x_{\mathrm{unobs}}$, the following limits are valid:
\begin{align}
     \label{eq:lower}
    l(x_{\mathrm{unobs}}) &:= \max \Big\{0, \sup_{x_{\mathrm{obs}} \in \mathcal X^{\mathrm{obs}}_{1:t}} \alpha(x_{\mathrm{obs}}) -  L_x\|x_{\mathrm{obs}}-x_{\mathrm{unobs}}\|_2 \Big\}  \le \alpha(x_{\mathrm{unobs}})  \\
    u(x_{\mathrm{unobs}}) &:= \min \Big\{1, \inf_{x_{\mathrm{obs}} \in \mathcal X^{\mathrm{obs}}_{1:t}} \alpha(x_{\mathrm{obs}}) + L_x\|x_{\mathrm{obs}}-x_{\mathrm{unobs}}\|_2 \Big\} \ge \alpha(x_{\mathrm{unobs}}) \label{eq:upper}
\end{align}

These two limits, while depending on quantities $\alpha$, could be estimated from the observed data, as for an accepted sample such that $O = 1$, we have that $P(Y=1|X=x_{\mathrm{obs}})= P(Y=1|X=x_{\mathrm{obs}}, O = 1)$. This definition of the bounds $l$ and $u$ gives a more specific expression for Theo.~\ref{theo:model_gap} that relates the sensitivity of $\alpha$ and the coverage of observation induced by previous models $\theta_1, \dots, \theta_t$.

\begin{theorem}
\label{theo:coverage_gap_invariant}
Assume the conditions of Theo.~\ref{theo:model_gap}, that $\alpha(x)$ is $(L_x, 0)$-sensitive, and define $l$ and $u$ as in Eqs.~\ref{eq:lower}--\ref{eq:upper}. Then, for any deployed
model $\phi$,
\begin{align}
\|G(\phi)-\overline G(\phi)\|_2
\le
\frac{2L_x}{\gamma}
\sup_{\theta\in\Theta}
\mathbb E_{D(\phi)}
\left[
(1-O)
\inf_{x_{\mathrm{obs}}\in\mathcal X^{\mathrm{obs}}_{1:t}}
\|X-x_{\mathrm{obs}}\|_2
\,
\|\nabla_\theta\Delta_\ell(X,\theta)\|_2
\right].
\end{align}
In particular, for binary cross-entropy with linear score
$f_\theta(x)=\theta^\top x$ and $\|x\|_2\le R$,
\begin{align}
\|G(\phi)-\overline G(\phi)\|_2
\le
\frac{2L_xR}{\gamma}
\mathbb E_{D(\phi)}
\left[
(1-O)
\inf_{x_{\mathrm{obs}}\in\mathcal X^{\mathrm{obs}}_{1:t}}
\|X-x_{\mathrm{obs}}\|_2
\right].
\end{align}
\end{theorem}

This result shows that the gap between $ G$ and $ \overline G$ depends on the sensitivity constant $L_x$ and the coverage of the feature space $\mathcal X$ by previous models $\theta_1, \dots, \theta_t$. This shows that the gap is smaller when $\alpha(x)$ changes little for similar features $x$ and when the models $\theta_t$ cover a large region of the feature space. Additionally, for a fixed deployed model $\phi$, the bound can only tighten over iterations, as $\mathcal X_{1:t}^{\mathrm{obs}}$ can only grow.

\section{Experiments}
\label{sec:experiments}

We evaluate our framework in a lending application with fairness regularization. This setting was introduced by \citet{liu2018delayed} and studied in long-term fairness work \citep{valdrighi2026longterm, damour2020fairness}, but we adapt it to the performative setting. In this section, we evaluate Robust RRM when $\alpha(x)$ is independent of $\theta$ and present experiments with $L_\theta \neq 0$ in Appendix~\ref{app:experiment_outcome}. Our experiments were designed to answer three main questions regarding Robust RRM: 1) if the performative gap $J(\theta_{t+1}, \theta_{t+1}) - J(\theta_{t}, \theta_{t+1})$ decreases when $t$ increases; 2) whether there are improvements in approximation to the stable solution $\theta^{PS}$ by using the proposed methodology to improve $l$ and $u$ during learning; and 3) if Robust RRM approximates RRM better than compared baselines. To isolate Robust RRM convergence, we evaluate questions 1-3 assuming knowledge of $L_x$. Finally, we evaluate whether Robust RRM performance is robust to misspecifications of $L_x$, and in Appendix~\ref{app:estimating} we evaluate when $L_x$ is estimated from data.

\paragraph{Baselines}
 We compare Robust RRM to RRM with Oracle access to labels. This baseline serves as a reference for attainable performance with full label access and to obtain $\theta^{PS}$. Additionally, we compare our framework to the naive procedure of optimizing only on the selected data, following $\tilde G$ (Accept-only). We also compare our algorithm with baselines from the literature: Affine Risk Minimization (Affine RM) \citep{khorsandi2025tight} --- which was proposed to improve convergence rates in performative prediction; an expectation-maximization algorithm for selective labels (DCEM) \citep{chang2024biased}; and a doubly robust estimator (DR Estimator) --- two baselines that handle partial observation without addressing performativity. See Appendix~\ref{app:implementation} for implementation details of each agent.

\paragraph{Experimental Setup}
We implemented a simulated environment where, at each round, the decision-maker obtained access to $n=100{,}000$ observations of variables $\{x_i, y_i \cdot o_i, o_i\}$. We use a large sample to reduce variance when estimating empirical losses. From this dataset, the decision-maker learns parameters $\theta \in \mathbb R^d$ of a logistic regression model. The model computes the linear score $f_\theta(x) = \theta^\top x$ and is trained with the binary cross-entropy loss with $\sigma(f_\theta(x))$ as the probability of the positive label. Applicants are accepted when $f_\theta(x) \ge \tau$ with $\tau = 0$, equivalent to $\sigma(f_\theta(x)) \ge 0.5$. The environment then receives $\theta$ to output the updated distribution $\{x_i', y_i' \cdot o_i', o_i'\}$. The process repeats for 400 iterations, where at each step the decision-maker updates $\theta$ to minimize risk under the updated data distribution according to each agent's objective. We approximate the risk minimizer using gradient descent for 5,000 steps or until the loss decreases by less than $10^{-7}$ in one iteration.

\subsection{Fair Lending}

In this experiment, we validate our worst-case approach on the long-term fairness environment introduced by \citet{liu2018delayed}, which \citet{valdrighi2026longterm} have already discussed with selective labels. This environment considers that model decisions will alter an individual's credit score, and the decision-maker should satisfy fairness in this new data distribution.

We leverage a continuous adaptation of the setting. Individuals are described by $x = (c, z, 1)$, where $c \in \mathbb R$ is the credit score and $ z$ is a binary sensitive attribute indicating whether individuals self-identify as Black or white and the last is used for the intercept. After deploying model $\theta$, the credit score will be updated by the rule  $c_{\theta} = c + \epsilon (\alpha(x) - 0.5) \cdot \sigma(f_{\theta}(x))$. Whenever $\alpha(x) > 0.5$, the credit score will increase proportionally to the predicted probability $\sigma(f_\theta(x))$, and when $\alpha(x) < 0.5$, it will decrease. Additionally, in this environment, the decision-maker should satisfy demographic parity, which enforces an equal acceptance rate between groups \citep{mehrabi2021survey}, that is, $DP(\phi,\theta) := |\E_{D(\phi)}[\sigma(f_\theta(X)) | Z = 1] - \E_{D(\phi)}[\sigma(f_\theta(X)) | Z = 0]| $. We leverage a continuous and differentiable proxy loss as follows:
\begin{align}
    J_\lambda^{\mathrm{fair}}(\phi, \theta) := J(\phi, \theta) + \lambda (\E_{D(\phi)} [\theta^\top X | Z = 1] -  \E_{D(\phi)} [\theta^\top X | Z = 0])^2
\end{align}
where $\lambda$ is a weight set to $0.3$ in our experiment. This extra term depends only on predicted scores, not on $Y$, and is unaffected by selective labels. We calculate the initial distribution of sensitive groups and credit scores using the FICO \citep{fico} dataset and fit a linear interpolation to simulate $\alpha(\cdot)$. For a detailed description, see Appendix~\ref{app:fico}. We obtain the data distribution $D(\cdot)$ by first updating all credit scores $c$ using the described rule, then resampling $y$ based on $\alpha(x)$. The label distribution satisfies that $L_x = 0.15$.

\begin{figure}[t]
  \centering
  \begin{subfigure}[b]{0.36\textwidth}
    \centering
    \includegraphics[width=\linewidth]{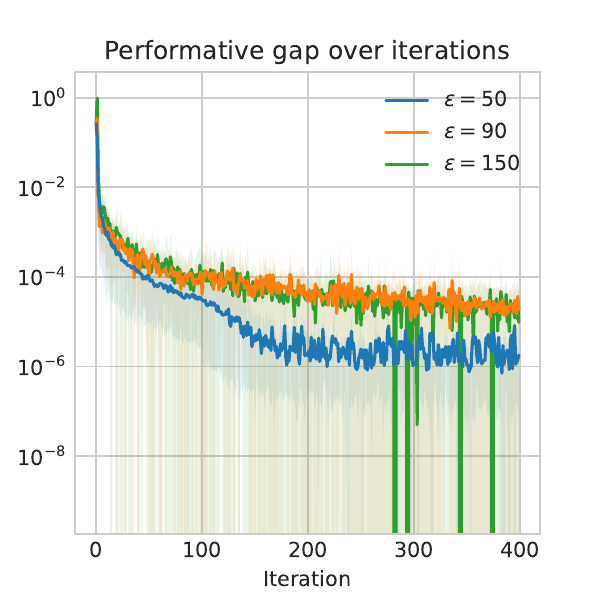}
    \caption{}
    \label{fig:performative_gap}
  \end{subfigure}\hfill
  \begin{subfigure}[b]{0.62\textwidth}
    \centering
    \includegraphics[width=\linewidth]{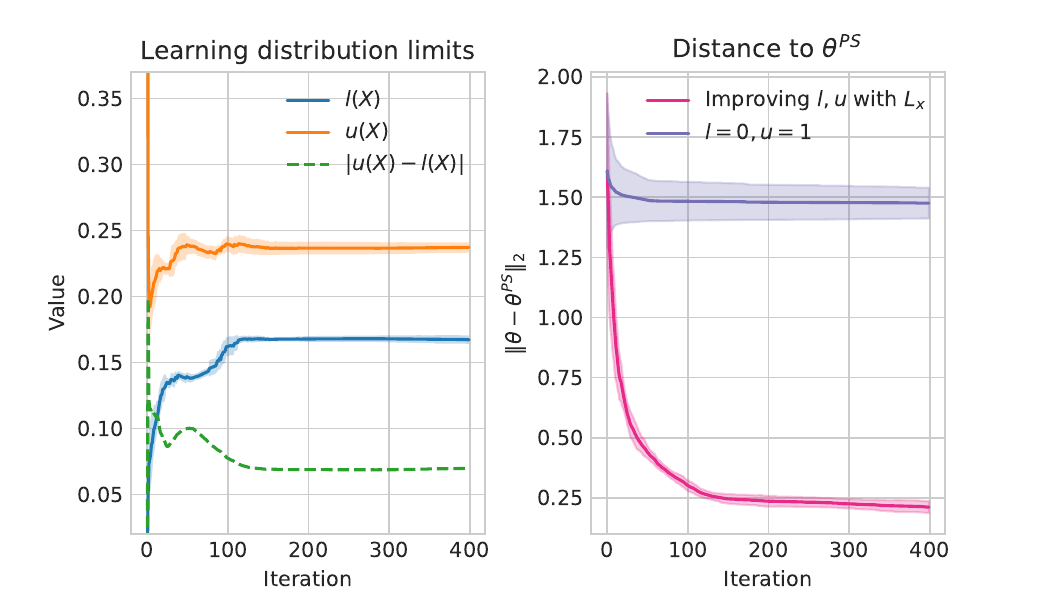}
    \caption{}
    \label{fig:analysis_2}
  \end{subfigure}
  \caption{Evaluation of Robust RRM. \textbf{(a)} Performative gap across different sensitivity levels of the distribution map (averages and 90\% interval over 10 random seeds). \textbf{(b)} Comparison of executing Robust RRM with trivial limits versus learned limits (averages and standard deviations over 10 random seeds).}
  \label{fig:robust_rrm_combined}
\end{figure}



\paragraph{Performative Gap}
As we discuss in Claim~\ref{claim:bound}, the proposed loss $\overline \ell_\phi$ is not necessarily $\beta$-smooth; therefore, it is not straightforward to prove that $\overline G$ is a contraction with a stable point. For this reason, we empirically evaluate whether the proposed approach yields a small performative gap, that is, the increase in loss after deploying $\theta_{t+1}$ in the distribution it induces, compared to the training distribution of $\theta_t$. In Fig.~\ref{fig:performative_gap}, we present the evolution of the performative gap over training for three different sensitivity constants $\epsilon \in \{50, 90, 150\}$. In all configurations, the performative gap reaches below $10^{-4}$, presenting the lowest values when $\epsilon$ is smaller. However, the results differ little between $\epsilon = 90$ and $\epsilon = 150$.

\paragraph{Improvements of Limits}
To analyze the procedure proposed in Sec.~\ref{sec:improving}, we compare two variants of Robust RRM: using the trivial limits $l(x) = 0, u(x)=1$ and updating $l, u$ based on $L_x$ and historical accepted data $\mathcal X^{\mathrm{obs}}_{1:t}$. Here, we show the average values of $l(X)$ and $u(X)$ for the rejected population, that is, $O=0$. We obtain these results with $\epsilon = 50$. As depicted in Fig.~\ref{fig:analysis_2}, the average width of the interval $|u(x) - l(x)|$ quickly decreases to values around $0.1$, decreasing during learning to reach values of $0.08$. At the end of $400$ iterations, the average values of $u$ and $l$ are around $0.24$ and $0.16$, respectively, reflecting the lower probability of the positive label for the rejected population. To compute $\theta^{PS}$, we run the Oracle agent for 400 iterations and select the final $\theta_t$. This improvement in interval width is reflected in how close the solutions $\theta_t$ are to $\theta^{PS}$. As shown in Fig.~\ref{fig:analysis_2}, using the trivial limits keeps the distance around $1.5$, while it decreases during learning to $0.212$ with tighter limits $l, u$.

\paragraph{Comparative Results}
We compare Robust RRM against baselines. In this experiment, we use the methodology to improve $l$ and $u$ and set $\epsilon = 50$. In Fig.~\ref{fig:credit_50}, we display the risk $J(\theta_t, \theta_t)$ calculated with the complete information of $Y$ and the demographic parity $DP(\theta_t, \theta_t)$. These quantities capture loss and unfairness after deployment, including performative effects. Additionally, we present the distance to the stable solution $\theta^{PS}$ of RRM with Oracle access to labels. The Robust RRM agent approximated the Oracle's performance and had the smallest distance to the stable solution. As predicted by Theo.~\ref{theo:convergence}, the solutions $\theta_t$ obtained by Robust RRM had a bounded distance to the true stable solution $\theta^{PS}$, which decreased over iterations. Other agents had larger distances to the stable solution. Affine RM performed almost identically to Accept-only; this indicates that using data from previous deployments does not help with selective labels. The DR Estimator agent diverged, presenting a very high performative risk and disparity. This may occur because the necessary assumptions for such an agent do not hold with thresholded decisions. Similar results are obtained with $\epsilon \in \{90, 150\}$ and are presented in Appendix~\ref{app:extra_results}.

\begin{figure}
    \centering
    \includegraphics[width=\linewidth]{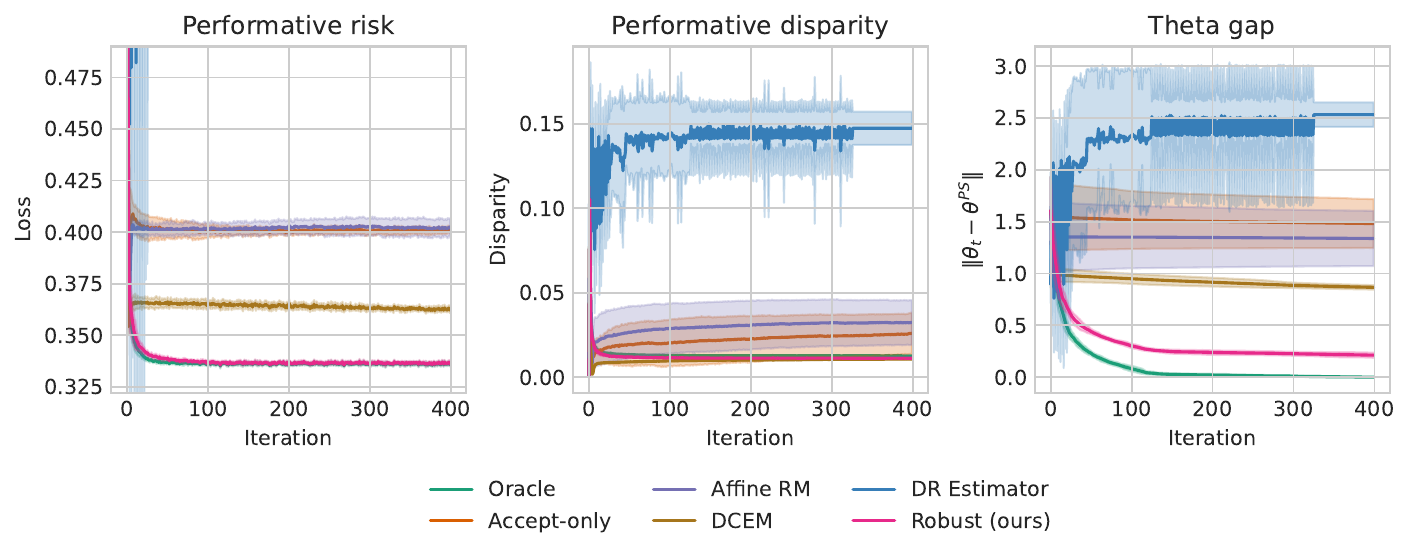}
    \caption{Loss, disparity, and distance to the stable solution for RRM with different objective configurations using $\epsilon = 50$. The y-axis was cropped to improve readability and does not display the full results for the DR Estimator agent. Results are averages and standard deviations over 10 random seeds.}
    \label{fig:credit_50}
\end{figure}

\paragraph{Robustness to the Misspecification of $L_x$}
To evaluate robustness to misspecification of $L_x$, we ran an additional experiment in which we perturbed the $L_x$ value used by Robust RRM by using $L_x+c$ for $c \in \{-0.075, -0.0375, 0.0375, +0.075\}$. Tab.~\ref{tab:robustness} reports results after $400$ retraining iterations. Interestingly, underestimating $L_x$ produced a smaller empirical distance to the stable solution. A smaller value of $L_x$ produces tighter intervals $|u(x)-l(x)|$ and therefore a smaller bound on the gap between the RRM and Robust RRM updates (Theo.~\ref{theo:coverage_gap_invariant}), which in turn reduces $\kappa$ in Theo.~\ref{theo:convergence}. However, when $L_x$ is underestimated, it is not guaranteed that $\alpha(x)\in[l(x),u(x)]$ for all $x$, and the original theoretical guarantee may not hold globally. Empirical results suggest that the limits remain valid for most of the feature space.

\begin{table}[t]
\centering
\caption{Results of Robust RRM with misspecified $L_x$. Results are averages (and standard deviations) over 10 random seeds. The distance to the stable solution reaches a maximum of $0.282$ with +50\% misspecification.}
\label{tab:robustness}
\begin{tabular}{lcc}
\hline
Constant used (\%misspecification)  & Performative risk $(\downarrow)$ & Dist. to $\theta^{PS}$  $(\downarrow)$ \\
\hline
$L_x =0.15$ correctly specified & $0.336 (\pm 0.001)$ & $0.212 (\pm 0.025)$ \\  \hline
$L_x + 0.0375$ (+25\%) & $0.336 (\pm 0.001)$ & $0.241 (\pm 0.020)$ \\  
$L_x - 0.0375$ (-25\%) & $0.336 (\pm 0.001)$ & $0.187 (\pm 0.018)$ \\ 
$L_x + 0.075$ (+50\%) & $0.336 (\pm 0.001)$ & $0.282 (\pm 0.028)$ \\ 
$L_x - 0.075$ (-50\%) & $0.336 (\pm 0.001)$ & $0.159 (\pm 0.024)$ \\ 
\hline
\end{tabular}
\end{table}

\section{Discussion}

\paragraph{Intervention and Observation}
In our framework, observation is defined by the final binary decisions obtained from thresholding predicted scores, while the model parameters define the population response. As \citet{mofakhami2023performative} note, individuals typically have access only to final decisions, not model parameters, which motivates defining distribution shift as a consequence of predictions rather than parameters, consistent with our definition of observation. However, the discontinuity introduced by thresholding can break the distribution map's $\epsilon$-sensitivity, and this setting has not yet been widely studied. In our experiments in the lending environment, we introduce a distribution map that approximates the scenario where the distribution shift is only dependent on the final prediction, by defining it as a function of $f_\theta(x)$. While simple, this environment serves as an example in which performativity is related to final decisions.

\paragraph{Performative Optimality}
We present theoretical results regarding the stability of robust iterates. However, as \citet{miller2021outside} discuss, the stable solution can be far from optimal. We opt to focus on achieving stable solutions, as optimality depends on multiple model deployments and exploration, which typically requires access to a simulator of the world dynamics. With access to such a simulator, partial label observation is not an issue. The social cost of deploying suboptimal models to gather data and achieve optimality is an interesting direction for future work.

\paragraph{Limitations}
Our guarantees are stated at the population level; deriving finite-sample bounds would require separately characterizing the estimation error of the decoupled risk and of the label distribution learned from selectively observed data. Moreover, the limits in Sec.~\ref{sec:improving} require the sensitivity constants $L_x$ and $L_\theta$. Under deterministic threshold policies, these constants are not identifiable from accepted samples alone. Our experiments in Sec.~\ref{sec:experiments} isolate the behavior of Robust RRM by assuming $L_x$ is known, evaluate its robustness to misspecification of $L_x$, and, in Appendix~\ref{app:estimating}, use an estimate obtained from a small set of randomly accepted applicants. Finally, we focus on deterministic threshold policies; stochastic policies can increase label coverage and improve utility and fairness under selective labels \citep{kilbertus2020fair}, and extending our framework to such policies is a promising direction for future work.

\section{Conclusion}

In this paper, we studied the intersection of performative prediction and selective labels, two properties common in many real-world machine learning applications. We focus on the convergence of RRM under partial-label observation with a deterministic threshold, and show that this is sufficient to break convergence to the stable point, even in the absence of performative effects. This motivated a robust formulation of the decoupled risk based on partial knowledge of the label distribution. Our approach leverages the loss function's linearity in the positive-label probability to obtain a strongly convex upper bound that does not depend on the true label. Our robust approach uses this upper bound when we have no access to the true label. When we use it for RRM, we show that it approximates the stable point up to a distance related to the acceptance rate, the convexity of the loss function, and partial knowledge of the label distribution. Our last result shows how to improve our knowledge of the label distribution by leveraging previously accepted samples and their feature sensitivity. We evaluated our approach in a fair lending environment, where it approximated the Oracle agent's performance.

\section*{Acknowledgments}
This project was supported, in part, by the Brazilian Ministry of Science, Technology, and Innovation, with resources from Law nº 8,248, of October 23, 1991, within the scope of PPI-SOFTEX, coordinated by Softex, and published in Arquitetura Cognitiva (Phase 3), DOU 01245.003479/2024 -10, and by the São Paulo Research Foundation (FAPESP), Brazil, process number \#2024/17292-9. 

This work has been partially supported by the project “Society-Aware Machine Learning: The paradigm shift demanded by society to trust machine learning,” funded by the European Union and led by IV (ERC-2021-STG, SAML, 101040177). Views and opinions expressed are, however, those of the author(s) only and do not necessarily reflect those of the European Union or the European Research Council Executive Agency. Neither the European Union nor the granting authority can be held responsible for them.

\paragraph{Gen AI Usage Statement}
The core ideas, proposed methodologies, and experiments were developed solely by the authors. AI tools such as ChatGPT and Grammarly were used only to revise parts of the text for grammatical correctness and improve presentation.

\newpage

\bibliography{paper}

@inproceedings{brown2022performative,
  title={Performative Prediction in a Stateful World},
  author={Brown, Gavin and Hod, Shlomi and Kalemaj, Iden},
  booktitle={International Conference on Artificial Intelligence and Statistics},
  pages={6045--6061},
  year={2022},
  organization={PMLR}
}

@inproceedings{perdomo2020performative,
  title={Performative Prediction},
  author={Perdomo, Juan and Zrnic, Tijana and Mendler-D{\"u}nner, Celestine and Hardt, Moritz},
  booktitle={International Conference on Machine Learning},
  pages={7599--7609},
  year={2020},
  organization={PMLR}
}

@inproceedings{liu2018delayed,
  title={Delayed Impact of Fair Machine Learning},
  author={Liu, Lydia T and Dean, Sarah and Rolf, Esther and Simchowitz, Max and Hardt, Moritz},
  booktitle={International Conference on Machine Learning},
  pages={3150--3158},
  year={2018},
  organization={PMLR}
}

@inproceedings{rateike2024designing,
  title={Designing Long-term Group Fair Policies in Dynamical Systems},
  author={Rateike, Miriam and Valera, Isabel and Forr{\'e}, Patrick},
  booktitle={ACM Conference on Fairness, Accountability, and Transparency},
  pages={20--50},
  year={2024}
}

@inproceedings{valdrighi2026longterm,
title={Long-term Fairness with Selective Labels},
author={Giovani Valdrighi and Isabel Valera and Marcos M. Raimundo},
booktitle={International Conference on Machine Learning},
year={2026},
}

@inproceedings{zezulka2024fair,
  title={From the Fair Distribution of Predictions to the Fair Distribution of Social Goods: Evaluating the Impact of Fair Machine Learning on Long-term Unemployment},
  author={Zezulka, Sebastian and Genin, Konstantin},
  booktitle={ACM Conference on Fairness, Accountability, and Transparency},
  pages={1984--2006},
  year={2024}
}

@inproceedings{mofakhami2023performative,
  title={Performative Prediction with Neural Networks},
  author={Mofakhami, Mehrnaz and Mitliagkas, Ioannis and Gidel, Gauthier},
  booktitle={International Conference on Artificial Intelligence and Statistics},
  pages={11079--11093},
  year={2023},
  organization={PMLR}
}

@inproceedings{hardt2016strategic,
  title={Strategic classification},
  author={Hardt, Moritz and Megiddo, Nimrod and Papadimitriou, Christos and Wootters, Mary},
  booktitle={ACM Conference on Innovations in Theoretical Computer Science},
  pages={111--122},
  year={2016}
}

@article{li2024stochastic,
  title={Stochastic optimization schemes for performative prediction with nonconvex loss},
  author={Li, Qiang and Wai, Hoi-To},
  journal={Advances in Neural Information Processing Systems},
  volume={37},
  pages={8673--8697},
  year={2024}
}

@article{hardt2025performative,
  title={Performative prediction: Past and future},
  author={Hardt, Moritz and Mendler-D{\"u}nner, Celestine},
  journal={Statistical Science},
  volume={40},
  number={3},
  pages={417--436},
  year={2025},
  publisher={Institute of Mathematical Statistics}
}

@article{kehrenberg2026dissecting,
  title={Dissecting Performative Prediction: A Comprehensive Survey},
  author={Kehrenberg, Thomas and Bautiste, Francisco Javier Sanguino and Lozano, Jose and Quadrianto, Novi},
  journal={ACM Computing Surveys},
  volume={58},
  number={13},
  pages={1--35},
  year={2026},
  publisher={ACM New York, NY}
}

@inproceedings{kilbertus2020fair,
  title={Fair Decisions Despite Imperfect Predictions},
  author={Kilbertus, Niki and Rodriguez, Manuel Gomez and Sch{\"o}lkopf, Bernhard and Muandet, Krikamol and Valera, Isabel},
  booktitle={International Conference on Artificial Intelligence and Statistics},
  pages={277--287},
  year={2020},
  organization={PMLR}
}

@article{dunner2022anticipating,
  title={Anticipating Performativity by Predicting from Predictions},
  author={Mendler-D{\"u}nner, Celestine and Ding, Frances and Wang, Yixin},
  journal={Advances in Neural Information Processing Systems},
  volume={35},
  pages={31171--31185},
  year={2022}
}

@inproceedings{lakkaraju2017selective,
  title={The Selective Labels Problem: Evaluating Algorithmic Predictions in the Presence of Unobservables},
  author={Lakkaraju, Himabindu and Kleinberg, Jon and Leskovec, Jure and Ludwig, Jens and Mullainathan, Sendhil},
  booktitle={ACM SIGKDD International Conference on Knowledge Discovery and Data Mining},
  pages={275--284},
  year={2017}
}

@inproceedings{izzo2021learn,
  title={How to Learn When Data Reacts to Your Model: Performative Gradient Descent},
  author={Izzo, Zachary and Ying, Lexing and Zou, James},
  booktitle={International Conference on Machine Learning},
  pages={4641--4650},
  year={2021},
  organization={PMLR}
}

@inproceedings{izzo2022learn,
  title={How to Learn When Data Gradually Reacts to Your Model},
  author={Izzo, Zachary and Zou, James and Ying, Lexing},
  booktitle={International Conference on Artificial Intelligence and Statistics},
  pages={3998--4035},
  year={2022},
  organization={PMLR}
}

@inproceedings{liu2024two,
  title={Two-timescale Derivative Free Optimization for Performative Prediction with {Markovian} Data},
  author={Liu, Haitong and Li, Qiang and Wai, Hoi To},
  booktitle={International Conference on Machine Learning},
  pages={31425--31450},
  year={2024},
  organization={PMLR}
}

@misc{fico,
author={{US} Federal Reserve},
title={Report to the Congress on credit scoring and its effects on the availability and affordability of credit},
year={2007}
}

@article{harris2023strategic,
  title={Strategic apple tasting},
  author={Harris, Keegan and Podimata, Chara and Wu, Steven Z},
  journal={Advances in Neural Information Processing Systems},
  volume={36},
  pages={79918--79945},
  year={2023}
}

@inproceedings{hu2022achieving,
  title={Achieving long-term fairness in sequential decision making},
  author={Hu, Yaowei and Zhang, Lu},
  booktitle={AAAI Conference on Artificial Intelligence},
  volume={36},
  pages={9549--9557},
  year={2022}
}

@article{mishler2022fair,
  title={Fair when trained, unfair when deployed: Observable fairness measures are unstable in performative prediction settings},
  author={Mishler, Alan and Dalmasso, Niccol{\`o}},
  journal={arXiv preprint arXiv:2202.05049},
  year={2022}
}

@inproceedings{jin2026addressing,
  title={Addressing polarization and unfairness in performative prediction},
  author={Jin, Kun and Xie, Tian and Liu, Yang and Zhang, Xueru},
  booktitle={AAAI Conference on Artificial Intelligence},
  volume={40},
  pages={22408--22416},
  year={2026}
}

@book{mackenzie2007economists,
  title={Do Economists Make Markets?: On the Performativity of Economics},
  author={MacKenzie, Donald A and Muniesa, Fabian and Siu, Lucia},
  year={2007},
  publisher={Princeton University Press}
}

@article{mehrabi2021survey,
  title={A survey on bias and fairness in machine learning},
  author={Mehrabi, Ninareh and Morstatter, Fred and Saxena, Nripsuta and Lerman, Kristina and Galstyan, Aram},
  journal={ACM Computing Surveys (CSUR)},
  volume={54},
  number={6},
  pages={1--35},
  year={2021},
  publisher={ACM New York, NY, USA}
}

@article{mendler2020stochastic,
  title={Stochastic optimization for performative prediction},
  author={Mendler-D{\"u}nner, Celestine and Perdomo, Juan and Zrnic, Tijana and Hardt, Moritz},
  journal={Advances in Neural Information Processing Systems},
  volume={33},
  pages={4929--4939},
  year={2020}
}

@inproceedings{li2022state,
  title={State dependent performative prediction with stochastic approximation},
  author={Li, Qiang and Wai, Hoi-To},
  booktitle={International Conference on Artificial Intelligence and Statistics},
  pages={3164--3186},
  year={2022},
  organization={PMLR}
}

@inproceedings{jagadeesan2022regret,
  title={Regret minimization with performative feedback},
  author={Jagadeesan, Meena and Zrnic, Tijana and Mendler-D{\"u}nner, Celestine},
  booktitle={International Conference on Machine Learning},
  pages={9760--9785},
  year={2022},
  organization={PMLR}
}

@inproceedings{
li2025clipped,
title={Clipped {SGD} Algorithms for Performative Prediction: Tight Bounds for Stochastic Bias and Remedies},
author={Qiang Li and Michal Yemini and Hoi To Wai},
booktitle={International Conference on Machine Learning},
year={2025},
}

@inproceedings{mandal2023performative,
  title={Performative reinforcement learning},
  author={Mandal, Debmalya and Triantafyllou, Stelios and Radanovic, Goran},
  booktitle={International Conference on Machine Learning},
  pages={23642--23680},
  year={2023},
  organization={PMLR}
}

@article{yan2023zero,
  title={Zero-regret performative prediction under inequality constraints},
  author={Yan, Wenjing and Cao, Xuanyu},
  journal={Advances in Neural Information Processing Systems},
  volume={36},
  pages={1298--1308},
  year={2023}
}

@inproceedings{miller2021outside,
  title={Outside the echo chamber: Optimizing the performative risk},
  author={Miller, John P and Perdomo, Juan C and Zrnic, Tijana},
  booktitle={International Conference on Machine Learning},
  pages={7710--7720},
  year={2021},
  organization={PMLR}
}

@inproceedings{zhao2025online,
  title={Online Strategic Classification With Noise and Partial Feedback},
  author={Zhao, Tianrun and Mao, Xiaojie and Liang, Yong},
  booktitle={Advances in Neural Information Processing Systems},
  year={2025},
}

@inproceedings{chen2025learning,
title={Learning with Selectively Labeled Data from Multiple Decision-makers},
author={Jian Chen and Zhehao Li and Xiaojie Mao},
booktitle={International Conference on Machine Learning},
year={2025},
}

@inproceedings{damour2020fairness,
  title={Fairness is not static: deeper understanding of long term fairness via simulation studies},
  author={D'Amour, Alexander and Srinivasan, Hansa and Atwood, James and Baljekar, Pallavi and Sculley, David and Halpern, Yoni},
  booktitle={ACM Conference on Fairness, Accountability, and Transparency},
  pages={525--534},
  year={2020}
}

@article{pombal2022prisoners,
  title={Prisoners of their own devices: How models induce data bias in performative prediction},
  author={Pombal, Jos{\'e} and Saleiro, Pedro and Figueiredo, M{\'a}rio AT and Bizarro, Pedro},
  journal={arXiv preprint arXiv:2206.13183},
  year={2022}
}

@inproceedings{rateike2022don,
  title={Don’t throw it away! the utility of unlabeled data in fair decision making},
  author={Rateike, Miriam and Majumdar, Ayan and Mineeva, Olga and Gummadi, Krishna P and Valera, Isabel},
  booktitle={ACM Conference on Fairness, Accountability, and Transparency},
  pages={1421--1433},
  year={2022}
}

@inproceedings{fawkes2024the,
	title        = {The Fragility of Fairness: Causal Sensitivity Analysis for Fair Machine Learning},
	author       = {Jake Fawkes and Nic Fishman and Mel Andrews and Zachary Chase Lipton},
	year         = {2024},
	booktitle    = {Advances in Neural Information Processing Systems Datasets and Benchmarks Track},
}

@inproceedings{keswani2024fair,
author = {Keswani, Vijay and Mehrotra, Anay and Celis, L. Elisa},
title = {Fair classification with partial feedback: an exploration-based data collection approach},
year = {2024},
booktitle = {International Conference on Machine Learning},
}

@inproceedings{frauen2024fair,
  title={Fair off-policy learning from observational data},
  author={Frauen, Dennis and Melnychuk, Valentyn and Feuerriegel, Stefan},
  booktitle={International Conference on Machine Learning},
  pages={13943--13972},
  year={2024}
}

@inproceedings{creager_causal_2020,
  title={Causal modeling for fairness in dynamical systems},
  author={Creager, Elliot and Madras, David and Pitassi, Toniann and Zemel, Richard},
  booktitle={International Conference on Machine Learning},
  pages={2185--2195},
  year={2020},
  organization={PMLR}
}

@InProceedings{kim2023making,
  author =	{Kim, Michael P. and Perdomo, Juan C.},
  title =	{Making Decisions Under Outcome Performativity},
  booktitle =	{Innovations in Theoretical Computer Science Conference},
  pages =	{79:1--79:15},
  series =	{Leibniz International Proceedings in Informatics (LIPIcs)},
  year =	{2023},
}

@inproceedings{khorsandi2025tight,
title={Tight Lower Bounds and Improved Convergence in Performative Prediction},
author={Pedram Khorsandi and Rushil Gupta and Mehrnaz Mofakhami and Simon Lacoste-Julien and Gauthier Gidel},
booktitle={Advances in Neural Information Processing Systems},
year={2025},
}

@inproceedings{chang2024biased,
author = {Chang, Trenton and Wiens, Jenna},
title = {From biased selective labels to pseudo-labels: an expectation-maximization framework for learning from biased decisions},
year = {2024},
booktitle = {International Conference on Machine Learning},
}

@inproceedings{boeken2024evaluating,
  title={Evaluating and correcting performative effects of decision support systems via causal domain shift},
  author={Boeken, Philip and Zoeter, Onno and Mooij, Joris},
  booktitle={Causal Learning and Reasoning},
  pages={551--569},
  year={2024},
  organization={PMLR}
}

@book{casella2002statistical,
  title={Statistical inference},
  author={Casella, George and Berger, Roger},
  year={2002},
  publisher={CRC press}
}

\newpage

\appendix

\section{Proofs}
\label{app:proofs}

\subsection{Proof of Claim~\ref{claim:bound}}

We prove the inequality by first proving that the loss:
\begin{align}
    \tilde \ell_\phi(x, \theta) := \max_{c \in \{l(x, \phi), u(x, \phi)\}} \ell(x, 0, \theta) + c \Delta_\ell(x, \theta)
\end{align}
satisfies $\E_{D(\phi)}[\ell(X, Y, \theta)] \le \E_{D(\phi)}[\tilde \ell_\phi(X, \theta)] $ and then that $\E_{D(\phi)}[\tilde \ell_\phi(X, \theta)] \le \E_{D(\phi)}[\overline \ell_\phi(X, \theta)]$. We prove each inequality separately. First, we can open the conditional expectation of the loss:
\begin{align}
    \mathbb E_{D(\phi)}[\ell(X, Y, \theta) | X =x] &= \alpha_\phi (x) \ell(x, 1, \theta)  +(1 - \alpha_\phi(x)) \ell(x, 0, \theta) \\
    &= \ell(x, 0, \theta) + \alpha_\phi(x) \Delta_\ell(x, \theta)
\end{align}
Since this expression is a linear function of $\alpha_\phi(x)$, and if AS.~\ref{as:limits} holds, then $\alpha_\phi(x) \in [l(x, \phi), u(x, \phi)]$. Thus:
\begin{align}
     \ell(x, 0, \theta) + \alpha_\phi(x) \Delta_\ell(x, \theta) \le \max_{c \in \{l(x, \phi), u(x, \phi)\}} \ell(x, 0, \theta) + c \Delta_\ell(x, \theta)
\end{align}
This shows that $\mathbb E_{D(\phi)}[\ell(X, Y, \theta) | X = x] \le \tilde \ell_\phi(x, \theta)$, which by integrating over the marginal distribution of $X$ permits us to conclude that:
\begin{align}
    \mathbb E_{D(\phi)}[\ell(X, Y,\theta)] \le \mathbb E_{D(\phi)}[\tilde \ell_\phi(X, \theta)]
\end{align}
Let $A := l(x, \phi) \Delta_\ell(x, \theta)$ and $B := u(x, \phi) \Delta_\ell(x, \theta)$. The second inequality can be proved with the following steps:
\begin{align}
\tilde \ell_\phi(x, \theta) &= \ell(x, 0, \theta) + \max \{A, B \} \\
&= \ell(x, 0, \theta) + \log(\max\{e^A, e^B\}) \\
 &< \ell(x, 0, \theta) + \log(e^A + e^B) \\
 &= \overline \ell_\phi(x, \theta)
\end{align}
The third line follows from the inequality $\max\{e^A, e^B\} < e^A + e^B$. Finally, as this result is valid for any $x, \theta$, we can integrate over the distribution of $X$ to obtain:
\begin{align}
    \mathbb E_{D(\phi)}[\tilde \ell_\phi(X,\theta)] \le \mathbb E_{D(\phi)}[\overline \ell_\phi(X, \theta)]
\end{align}
We now prove that $\overline \ell_\phi$ is $\gamma$-strongly convex. Fix $x$ and $\phi$, and write $a:=l(x,\phi), b:=u(x,\phi)$. For $c\in\{a,b\}$, define $g_c(\theta):=(1-c)\ell(x,0,\theta)+c\ell(x,1,\theta)$. Since $c\in[0,1]$ and $\ell(x,y,\theta)$ is $\gamma$-strongly convex in $\theta$ for each $y$, each $g_c$ is also $\gamma$-strongly convex. Hence
\begin{align}
    h_c(\theta):=g_c(\theta)-\frac{\gamma}{2}\|\theta\|_2^2
\end{align}
is convex. Moreover,
\begin{align}    
\overline\ell_\phi(x,\theta)    &= \log\left(e^{g_a(\theta)}+e^{g_b(\theta)}\right) \\    &= \frac{\gamma}{2}\|\theta\|_2^2    + \log\left(e^{h_a(\theta)}+e^{h_b(\theta)}\right).
\end{align}
The log-sum-exp of convex functions is convex, so $\overline\ell_\phi$ is $\gamma$-strongly convex in $\theta$.

The loss $\overline\ell_\phi$ is not necessarily $\beta$-smooth in $x$ because $ l$ and $ u$ may not be smooth. It suffices to give a counterexample. Let $\Theta=[-1,1]$, $\mathcal X=[-1,1]$, and\[\ell(x,0,\theta)=\frac{\gamma}{2}\theta^2,\qquad\ell(x,1,\theta)=\frac{\gamma}{2}\theta^2+\theta.\]Then $\Delta_\ell(x,\theta)=\theta$. Choose\[u(x,\phi)=1,\qquad l(x,\phi)=\mathbf 1\{x\ge 0\}.\]These functions satisfy AS.~\ref{as:limits}, for instance, with $\alpha_\phi(x) = 1$ for all $x$, but\[\nabla_\theta \overline\ell_\phi(x,0)=\begin{cases}1/2, & x<0,\\1, & x\ge 0.\end{cases}\]Taking $x_n=-1/n$ and $x'_n=0$, we get\[\left|\nabla_\theta\overline\ell_\phi(x_n,0)-\nabla_\theta\overline\ell_\phi(x'_n,0)\right|=\frac12,\qquad|x_n-x'_n|=\frac1n.\]Thus no finite $\beta$ can satisfy the smoothness inequality. Hence $\overline\ell_\phi$ is $\gamma$-strongly convex, but not $\beta$-smooth in general.

\subsection{Proof of Theo.~\ref{theo:model_gap}}

\begin{proof}
    As $f(\cdot) := J(\phi, \cdot)$ and $\overline f(\cdot) := \overline J(\phi, \cdot)$ are both $\gamma$-strongly convex functions then: 
\begin{align}
    \|G(\phi) - \overline G(\phi)\|_2 \le \dfrac{1}{\gamma} \sup_{\theta \in \Theta} \|\nabla_\theta f(\theta) - \nabla_\theta \overline f(\theta)\|_2
\end{align}
We have that the difference between gradients is:
\begin{align}
    \nabla_\theta f(\theta) -  \nabla_\theta \overline f(\theta) = \mathbb E_{D(\phi)}\left [(1 - O) (\nabla_\theta \ell(X, Y,\theta) - \nabla_\theta \overline \ell_\phi(X, \theta))  \right]
\end{align}
Since $O = \mathbf{1}\{f_\phi(X) \ge \tau\}$ is a deterministic function of $X$, we can condition on $X = x$ inside the expectation, and we have:
\begin{align}
    \mathbb E_{D(\phi)}[\nabla_\theta \ell(X, Y,\theta)|X=x] &= \nabla_\theta \ell(x, 0, \theta) + \alpha_\phi(x)  \nabla_\theta \Delta_\ell(x, \theta)  \\ 
    \mathbb E_{D(\phi)}[\nabla_\theta \overline \ell_\phi(X, \theta)|X=x] &= \nabla_\theta \ell(x, 0, \theta) +  \nabla_\theta\log(e^{l(x, \phi) \Delta_\ell(x, \theta)} + e^{u(x, \phi) \Delta_\ell(x, \theta)}) \\ 
    &= \nabla_\theta \ell(x, 0, \theta) + w_\phi(x, \theta) \nabla_\theta \Delta_\ell(x, \theta)
\end{align}
where 
\begin{align}
 w_\phi(x, \theta)=\dfrac{e^{l(x, \phi) \Delta_\ell(x, \theta)} l(x, \phi)  + e^{u(x, \phi) \Delta_\ell(x, \theta)}u(x, \phi)}{e^{l(x, \phi) \Delta_\ell(x, \theta)} + e^{u(x, \phi) \Delta_\ell(x, \theta)}}
\end{align}
And then we have that:
\begin{align}
    \nabla_\theta f(\theta) -  \nabla_\theta \overline f(\theta) &= \mathbb E_{D(\phi)} \left [ (1 - O)(\alpha_{\phi}(X) - w_\phi(X, \theta))   \nabla_\theta\Delta_\ell(X, \theta) \right ] \\
    \implies \|\nabla_\theta f(\theta) -  \nabla_\theta \overline f(\theta)\|_2 &\le \mathbb E_{D(\phi)} \left [ (1 - O)|\alpha_{\phi}(X) - w_\phi(X, \theta)|  \| \nabla_\theta\Delta_\ell(X, \theta) \|_2\right ]
\end{align}
As $w_\phi(x, \theta)$ is a weighted mean of $l(x, \phi)$ and $u(x, \phi)$, then $w_\phi(x, \theta) \in [l(x, \phi), u(x, \phi)]$, then, $|\alpha_{\phi}(x) - w_\phi(x, \theta)| \le |u(x, \phi)  - l(x, \phi)|$ for any $x, \phi$. Finally, we can conclude that:
\begin{align}
        \|G(\phi) - \overline G(\phi)\|_2 \le \dfrac{1}{\gamma} \sup_{\theta \in \Theta} \mathbb E_{D(\phi)}[(1 - O)|u(X, \phi) - l(X, \phi)| \|\nabla_\theta \Delta_\ell(X, \theta) \|_2]
\end{align}
We analyze $\nabla_\theta \Delta_\ell(X, \theta)$ for common definitions of $\ell$:

\begin{itemize}
    \item Logistic regression: $\ell(x, y, \theta) := - \left ( y\log (\sigma(\theta^\top x))   + (1-y) \log(1 - \sigma(\theta^\top x)) \right )$ where $\sigma : \mathbb R \to (0, 1)$ is the logistic function. We then have $\Delta_\ell(x, \theta) = - \log(\sigma(\theta^\top x) / (1 - \sigma(\theta^\top x))) = -\theta^\top x$ and then $\|\nabla_\theta \Delta_\ell    (x, \theta)\|_2 = \|x\|_2$.
    \item Mean-squared error: $\ell(x, y, \theta) := (y - \theta^\top x)^2$. Then, we have that $\Delta_\ell(x, \theta) = (1 - \theta^\top x)^2 - (0 - \theta^\top x)^2=1 -2\theta^\top x$ and $\nabla_\theta \Delta_\ell(x, \theta) = -2(1 - \theta^\top  x) x - 2(\theta^\top x)x = -2x$ and $\|\nabla_\theta \Delta_\ell(x, \theta)\|_2 = 2\|x\|_2$.
\end{itemize}
Note that adding a regularization term $\|\theta\|_2^2$ to ensure that $\ell$ is $\gamma$-strongly convex leaves $\Delta_\ell (x, \theta)$ unchanged, since the penalization cancels out.
\end{proof}

\subsection{Proof of Theo.~\ref{theo:convergence}}

\begin{proof}
For any $\theta \in \Theta$, we can bound the distance between the worst-case update $\overline G(\theta)$ and the stable solution $\theta^{PS}$ using the triangle inequality and $G(\theta^{PS}) = \theta^{PS}$:
    \begin{align}
        \|\overline G(\theta) - \theta^{PS}\|_2 &\le \|\overline G(\theta) - G(\theta) \|_2 + \|G(\theta) - G(\theta^{PS})\|_2 \\
        & \le \kappa + \eta\|\theta - \theta^{PS}\|_2
    \end{align}
where $\eta := \epsilon \beta / \gamma$. Unrolling this recursion, the sequence defined as $\theta_{t+1} = \overline G(\theta_t)$ satisfies:
\begin{align}
    \| \theta_t - \theta^{PS}\|_2 \le \eta^t \| \theta_0 - \theta^{PS}\|_2 + \kappa \sum_{s=0}^{t-1} \eta^s = \eta^t \| \theta_0 - \theta^{PS}\|_2 + \dfrac{1- \eta^t}{1 - \eta} \kappa
\end{align}
where the right side converges to $\kappa/(1 - \eta)$ when $t \to \infty$, since $\eta < 1$.
\end{proof}

\subsection{Proof of Theo.~\ref{theo:coverage_gap_invariant}}

\begin{proof}
As $\alpha$ is $(L_x, 0)$-sensitive, for every $x_{\mathrm{obs}} \in \mathcal X^{\mathrm{obs}}_{1:t}$ we have $\alpha(x_{\mathrm{obs}}) - L_x\|x_{\mathrm{unobs}} - x_{\mathrm{obs}}\|_2 \le \alpha(x_{\mathrm{unobs}}) \le \alpha(x_{\mathrm{obs}}) + L_x\|x_{\mathrm{unobs}} - x_{\mathrm{obs}}\|_2$. Taking the supremum and infimum over $x_{\mathrm{obs}}$, and using $\alpha(x_{\mathrm{unobs}}) \in [0, 1]$ for the clipping, the limits in Eqs.~\ref{eq:lower}--\ref{eq:upper} satisfy $0 \le l(x_{\mathrm{unobs}})\le \alpha(x_{\mathrm{unobs}})\le u(x_{\mathrm{unobs}}) \le 1$, so AS.~\ref{as:limits} holds. Moreover, since clipping can only reduce $u(x_{\mathrm{unobs}}) - l(x_{\mathrm{unobs}})$, for any $x_{\mathrm{obs}}$ we have that:
\begin{align}
    \forall x_{\mathrm{obs}} \in \mathcal X^{\mathrm{obs}}_{1:t} , u(x_{\mathrm{unobs}})-l(x_{\mathrm{unobs}}) &\le 2L_x\|x_{\mathrm{unobs}}-x_{\mathrm{obs}}\|_2 \implies \\
    u(x_{\mathrm{unobs}})-l(x_{\mathrm{unobs}}) &\le 2L_x \inf_{x_{\mathrm{obs}}\in\mathcal X^{\mathrm{obs}}_{1:t}}\|x_{\mathrm{unobs}}-x_{\mathrm{obs}}\|_2
\end{align}
Substituting this bound on $|u(\cdot) - l(\cdot)|$ into Theo.~\ref{theo:model_gap} proves the general result and the special case for binary cross-entropy.
\end{proof}

\section{Bounding the Performative Risk of Robust RRM}
\label{app:bound_performative}

Theo.~\ref{theo:convergence} bounds the distance between iterations of Robust RRM and the stable solution in parameter space; however, it does not bound the performative risk of Robust RRM. We can bound the performative risk by combining Theo.~\ref{theo:convergence} with Lemma 2.1 of \citep{jagadeesan2022regret} that has extra assumptions on the loss function $\ell$. Define as $PR(\theta) := J(\theta, \theta)$ the performative risk. Then, the following lemma holds:

\begin{lemma}[Adapted from Lemma 2.1 of \citep{jagadeesan2022regret}]
\label{lemma:lipschitz}
If the loss $\ell(x, y, \theta)$ is $L_1-$Lipschitz in $(x, y)$ and $L_2-$Lipschitz in $\theta$, and the distribution map $D(\cdot)$ is $\epsilon-$sensitive, then the performative risk $PR(\cdot)$ is $(\epsilon L_1 + L_2)-$Lipschitz.
\end{lemma}

Now, we can use the Lipschitzness of $PR$ to bound the iterates of Robust RRM.

\begin{proposition}[Performative Risk of Robust RRM]
Let $\theta_t$ be an iterate of Robust RRM as in Def.~\ref{def:robust_rrm}. With the same assumptions as in Lemma~\ref{lemma:lipschitz} and from Theo.~\ref{theo:convergence}, then:
\begin{align}\limsup_{t \to \infty} PR(\theta_t) \le PR(\theta^{PS}) + (\epsilon L_1 + L_2) \frac{\kappa}{1 - \eta}
\end{align}
\end{proposition}

\begin{proof}
    The proof is direct from the application of the Lipschitzness of $PR(\cdot)$ and the bound on the distance $\| \theta_t - \theta^{PS}\|_2$:
    \begin{align}
        | PR(\theta_t ) - PR(\theta^{PS})| &\le (\epsilon L_1 + L_2) \| \theta_t  - \theta^{PS}\|_2 \\
        & \le (\epsilon L_1 + L_2) \left ( \eta^t\| \theta_0  - \theta^{PS}\|_2   + \dfrac{1 - \eta^t}{1 - \eta} \kappa\right)
    \end{align}
Moving $PR(\theta^{PS})$ to the right side and applying the limit $t \to \infty$ concludes the proof.
\end{proof}

While the bound may be conservative when $L_1$ or $L_2$  is large, our empirical evaluations consistently show that Robust RRM matches the performative risk attained by Oracle RRM.

\section{Improving Label Distribution Limits with Model Dependence}
\label{app:results_dependent}

In this section, we extend the methodology from Sec.~\ref{sec:improving} when the label is dependent on the deployed model $\theta$, that is, $L_\theta \neq 0$. We use similar steps to the ones presented in Sec.~\ref{sec:improving}. First, let $\mathcal X_s^{\mathrm{obs}}$ be the support of the features accepted at round $s$ and define $H_{1:t} := \bigcup_{s=1}^t (\mathcal X_s^{\mathrm{obs}} \times \{\theta_s\})$ as the set of previously accepted features $x$ paired with the respective deployed model. Then, let $\phi$ be the currently deployed model and $x_{\mathrm{unobs}}$ an unobserved sample. We define:
\begin{align}
    d(x_{\mathrm{unobs}}, x_{\mathrm{obs}}, \phi, \theta') :=& L_x \| x_{\mathrm{unobs}} - x_{\mathrm{obs}}\|_2 + L_\theta \| \phi - \theta' \|_2 \label{eq:distance} \\
    l(x_{\mathrm{unobs}}, \phi) := & \max \{ 0, \sup_{(x_{\mathrm{obs}}, \theta') \in H_{1:t}} \alpha_{\theta'}(x_{\mathrm{obs}}) - d(x_{\mathrm{unobs}}, x_{\mathrm{obs}}, \phi, \theta') \} \label{eq:lower_v2} \\
    u(x_{\mathrm{unobs}}, \phi) := & \min \{ 1, \inf_{(x_{\mathrm{obs}}, \theta') \in H_{1:t}} \alpha_{\theta'}(x_{\mathrm{obs}}) + d(x_{\mathrm{unobs}}, x_{\mathrm{obs}}, \phi, \theta') \} \label{eq:upper_v2}
\end{align}

Under these definitions, since $|\alpha_\phi(x_{\mathrm{unobs}}) - \alpha_{\theta'}(x_{\mathrm{obs}})| \le d(x_{\mathrm{unobs}}, x_{\mathrm{obs}}, \phi, \theta')$ for every $(x_{\mathrm{obs}}, \theta') \in H_{1:t}$, it holds that $ l(x_{\mathrm{unobs}}, \phi) \le   \alpha_\phi(x_{\mathrm{unobs}}) \le u(x_{\mathrm{unobs}}, \phi)$ for any unobserved sample and deployed model $\phi$. Notice that this formulation reduces to the definition of limits in Sec.~\ref{sec:improving} when $L_\theta = 0$. Next, to update Theo.~\ref{theo:coverage_gap_invariant}, it is only necessary to update the bound of $|u(\cdot, \cdot) - l(\cdot, \cdot)|$ as follows:
\begin{align}
    |u(x_{\mathrm{unobs}}, \phi) - l(x_{\mathrm{unobs}}, \phi)| \le 2L_x \|x_{\mathrm{unobs}}-x_{\mathrm{obs}}\|_2 +2 L_\theta \|\phi - \theta' \|_2, \quad \forall (x_{\mathrm{obs}}, \theta') \in H_{1:t}
\end{align}

\begin{theorem}
\label{theo:coverage_gap_dependent}
Assume the conditions of Theo.~\ref{theo:model_gap}, that $\alpha_\theta(x)$ is $(L_x, L_\theta)$-sensitive, and define $l,u$ as in Eqs.~\ref{eq:lower_v2}--\ref{eq:upper_v2} and the distance $d$ as in Eq.~\ref{eq:distance}. Then, for any deployed
model $\phi$,
\begin{align}
\|G(\phi)-\overline G(\phi)\|_2
\le
\frac{2}{\gamma}
\sup_{\theta\in\Theta}
\mathbb E_{D(\phi)}
\left[
(1-O)
\inf_{(x_{\mathrm{obs}}, \theta') \in H_{1:t}}
d(X, x_\mathrm{obs}, \phi, \theta') 
\,
\|\nabla_\theta\Delta_\ell(X,\theta)\|_2
\right].
\end{align}
In particular, for binary cross-entropy with linear score
$f_\theta(x)=\theta^\top x$ and $\|x\|_2\le R$,
\begin{align}
\|G(\phi)-\overline G(\phi)\|_2
\le
\frac{2R}{\gamma}
\mathbb E_{D(\phi)}
\left[
(1-O)
\inf_{(x_{\mathrm{obs}}, \theta') \in H_{1:t}}
d(X, x_\mathrm{obs}, \phi, \theta') 
\right].
\end{align}
\end{theorem}
In this result, we moved $L_x$ inside the distance function $d$. Despite this change, the interpretation of the results is the same: the one-step gap is dependent on the strong convexity of the loss function, the rejection rate, the dependence of the gradient of $\ell$ on $Y$, and the width of the limits $|u(\cdot, \cdot) - l (\cdot, \cdot)|$ that depends on the proximity of the current model $\phi$ and variable $X$ to the support of the historical acceptance $H_{1:t}$.

This approach to define $l$ and $u$ requires estimating the label distribution $\alpha$ within the accepted data. This is valid, as conditioned on acceptance, the label distribution is identifiable. However, in practice, when $L_\theta \neq 0$, we must estimate one $\alpha_\theta$ for each deployed $\theta$, because label probabilities differ for each previously used model $\theta$. Our experiments in Appendix~\ref{app:experiment_outcome} demonstrate the applicability of our approach in this more general setting.

\section{Extension to Multi-class Problems}
\label{app:multi_class}

In this section, we briefly discuss how to extend our approach to multi-class problems and the challenges it introduces. One initial consideration when moving from binary classification to multi-class classification is updating the thresholded definition of the observation variable $O$, as in the multi-class problem, the model does not predict a unidimensional score but $K$ scores for $K$ classes. One approach is to separate the predicted labels into two categories: labels with observation $O=1$ and unobserved labels. In a lending example, we could consider three risk classes (high, medium, low) and observe only the label (accepting) of applicants classified as medium or low risk.

The label distribution is defined for each class  $\alpha_\theta^k(x) = P(Y = k | X = x, \theta)$, as well as the limits $l^k(x, \theta), u^k(x, \theta)$. A worst-case loss can be obtained using the same ideas of the linearity of $\ell$ with respect to $\alpha_\theta^k$: 
\begin{align}
\overline \ell_\phi(x, \theta) := \sum_{k=1}^K \log \left (e^{l^k(x, \phi) \ell(x, k, \theta)} + e^{u^k(x, \phi)\ell(x, k, \theta)} \right)
\end{align}
Taking the expectation over $x$ upper-bounds the risk $\mathbb E[\ell(X, Y, \theta)]$. This worst-case loss follows from the fact that the conditional expectation of the loss over $x$ is a linear function of $\alpha^k_\theta(x)$, which attains its maximum at the interval endpoints. However, more work is necessary to study the effectiveness of this loss as it combines multiple maximum operations. A tighter bound could be obtained by using the constraint that $\sum_k \alpha_\theta^k (x) = 1$. However, simply adding such a constraint could break the strong convexity of the worst-case loss.

\section{Estimating \texorpdfstring{$L_x$}{Lx} and \texorpdfstring{$L_\theta$}{L-theta} from Data}
\label{app:estimating}

The proposed approach for improving limits $l$ and $u$ depends on knowing the constants $L_x$ and $L_\theta$, which control how much the probability of the positive label can change when the sample $x$ or model $\theta$ changes. In this section we discuss how such constants could be obtained. Under a deterministic acceptance rule defined by a threshold, both constants are generally not identifiable from accepted samples alone, as regions of the feature space will have zero probability of observation. Two label distributions may agree on the entire accepted region while disagreeing in the unobserved region. Estimating therefore requires domain knowledge, exploration, additional structure, or a combination of them. 

In low-dimensional applications with interpretable features, experts may provide bounds on how rapidly label probabilities can change. In the lending experiment, $x$ comprises the credit score and sensitive attribute, so $L_x$ represents an upper bound on how quickly repayment probability changes across nearby applicant profiles. Importantly, the expert need not identify the smallest Lipschitz constant: any valid conservative upper bound is sufficient.

Without domain knowledge, we can estimate the constants with a finite set of samples collected by randomly accepting applicants. Given additional structure to $\alpha_\theta(\cdot)$, we can obtain specific expressions for the estimator of the constants. We will give one example under the assumption that the label distribution is a logistic function, that is, $\alpha_\theta(x) := \sigma(w^{\star \top}_x x + w^{\star \top}_\theta \theta + b)$ for an unknown vector $w^{\star} := (w^\star_x, w^\star_\theta)$ where $w_x^\star$ contains the influence of features on the label distribution and $w_\theta^\star$ the influence of model parameters. 

By analyzing the maximum value of the gradient of $\alpha_\theta(x)$ we have that $L_x \le \| w^\star_x\|_2 /4$ and $L_\theta \le \| w^\star_\theta\|_2 /4$. It is therefore sufficient to obtain a conservative upper bound on the coefficient norm rather than identify the exact model. Let $\widehat w$ be the MLE of $w^\star$ and $p$ the failure probability. Under correct logistic specification and MLE regularity conditions, the Delta method \citep[Sec.~5.5.4]{casella2002statistical} gives the following asymptotic one-sided ($1 - p$)-confidence upper bound:
\begin{align}
    \widehat L_x := \dfrac{1}{4}\left (\| \widehat w_x \|_2 + z_{1-p} \sqrt{\dfrac{\widehat w_x^\top \widehat V\widehat w_x }{\| \widehat w_x \|_2^2}} \right) \quad \quad  \widehat V = (\mathbf X^\top \widehat W \mathbf X)^{-1}
\end{align}
where $z_{1-p}$ is the $1-p$ quantile of the normal,  $\mathbf X$ is the design matrix, $\widehat V$ is restricted to the coordinates of $w_x$, and $\widehat W$ is a diagonal matrix with $\widehat W[i, i] := s_i(1 - s_i)$ where $s_i$ is the predicted probability for sample $i$. An analogous expression is valid for $\widehat L_\theta$.

To demonstrate this procedure, we present in Tab.~\ref{tab:estimation_lx} 
the results of estimating $L_x = 0.15$ with the assumption that $\alpha_\theta(x) = \sigma(w_x^{\star\top} x + b)$ using a confidence of 95\% and different numbers of samples. To complement this, we ran the misspecification experiment with the overestimate $\widehat L_x = 0.287$ obtained by accepting 2,000 random samples independently of the threshold decision. The performative risk increased by only $0.001$ compared with using the correctly specified $L_x$, and the distance to the stable solution increased from $0.212 (\pm 0.025)$ to $0.363 (\pm 0.024)$, which is lower than results from compared agents. In the lending environment, the label distribution is not modeled by a logistic function, so the result should be interpreted as a data-derived heuristic rather than a calibrated constant $L_x$. However, the additional experiment tests whether Robust RRM remains effective when instantiated with a practical estimate.

\begin{table}[t]
\centering
\caption{Estimation of constant $L_x$ under the assumption of a logistic label distribution and randomized acceptance, independently of model decision.}
\label{tab:estimation_lx}
\begin{tabular}{lc}
\hline
    Number of samples & $\widehat L_x$ (and std. over 10 random seeds) \\ \hline
        200 & $0.457 (\pm 0.148)$ \\ 
        1000 & $0.331 (\pm 0.129)$ \\ 
        2000 & $0.287 (\pm 0.062)$ \\
        10000 & $0.207 (\pm 0.034)$ \\ \hline
\end{tabular}
\end{table}

\section{Implementation Details}
\label{app:implementation}

\paragraph{Computer Resources}
All experiments were executed on CPUs with 8GB of RAM, with no GPU use. Experiments ran on a cluster with automatic job assignment across AMD Ryzen Threadripper PRO 7975WX and Intel(R) Xeon(R) Silver 4410Y.

\subsection{Algorithms Implementation}

In our experiments, we do not have access to the true values of $J(\phi, \theta)$ (or other objectives), only empirical estimates. We retrain repeatedly using gradient descent on the empirical estimate of the objective. At each retraining round, we run 5,000 steps of gradient descent to find the model that approximately minimizes the objective, stopping early when the change in loss after one gradient update is below $10^{-7}$. We apply $L2$ regularization to the model parameters $\theta$ with a weight of $10^{-2}$ to ensure that the objective is strongly convex.

\paragraph{Robust RRM}
To calculate limits $l$ and $u$ we keep a buffer of observations $(\{x_i, y_i, \theta_i\})$ of maximum size $500,000$. We used this dataset to estimate $\alpha_\theta(x)$ within the observed feature space with a logistic regression model. Solving the maximization/minimization in the definitions of $l$ and $u$ involves evaluating distances $\|x_{\mathrm{obs}} - x_{\mathrm{unobs}}\|_2$ among observed samples $x_{\mathrm{obs}}$ and unobserved samples $x_{\mathrm{unobs}}$. To avoid computing every pair of $\|x_{\mathrm{obs}} - x_{\mathrm{unobs}}\|_2$, we used a KD-Tree and searched for the $50$ closest neighbors for each $x_{\mathrm{unobs}}$. Then, we evaluate the complete definitions of $l$ and $u$ among these $50$ samples. During our preliminary evaluation, increasing the neighborhood did not significantly change the results.

\begin{figure}[t]
    \centering
    \includegraphics[width=0.8\linewidth]{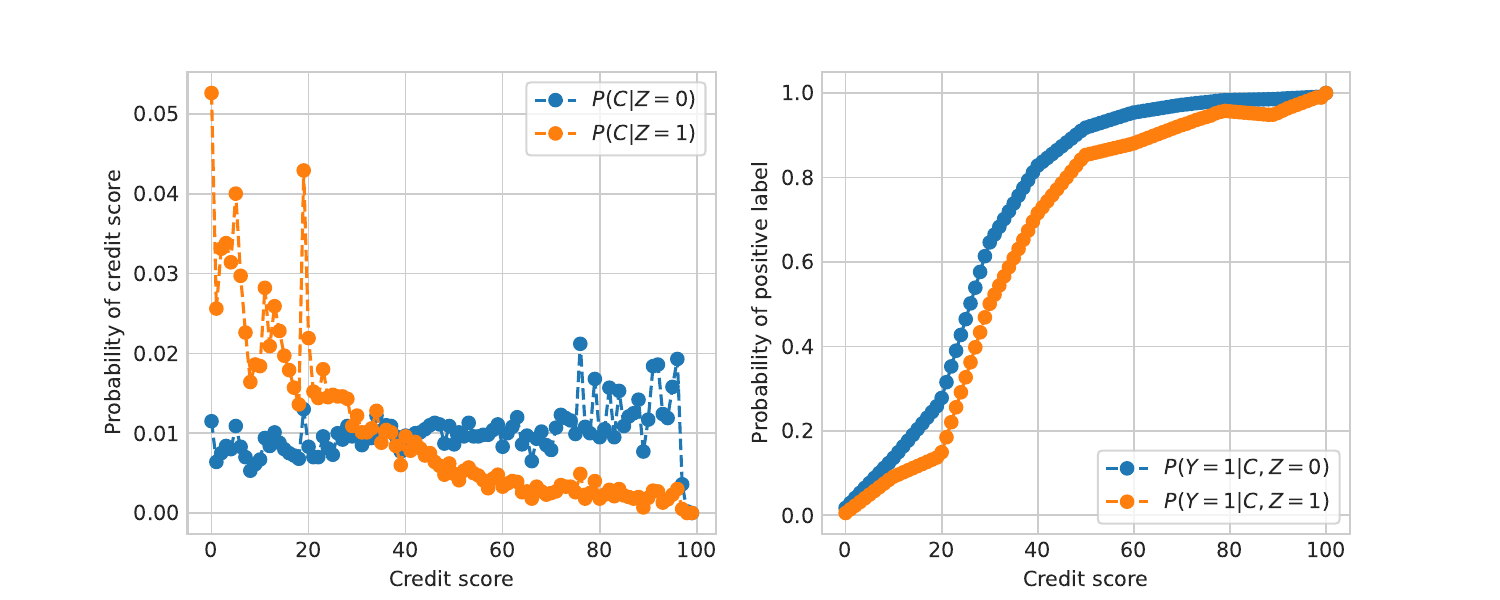}
    \caption{Distribution of scores and positive label calculated from the FICO dataset. $z=1$ represents individuals who self-identify as Black and $z=0$ individuals who self-identify as white.}
    \label{fig:fico}
\end{figure}

\paragraph{Oracle RRM, Accept-only and Affine RM}
We implemented all three variations of repeated risk minimization by changing only the data distribution the agent could access. Oracle RRM had access to the true distribution $D(\theta)$ and full label observations. Accept-only leverages the selected distribution $\tilde D(\theta)$ using only samples accepted by the previous model. Affine RM \citep{khorsandi2025tight} uses a larger historical dataset to minimize risk. We implemented this using the three previous iterations of selectively labeled data, that is, $\tilde D^3(\theta_t) := \frac{1}{3} \sum_{i = 0}^{2} \tilde D(\theta_{t-i})$. 

\paragraph{DR Estimator and DCEM}
These two agents leverage both the accept-only distribution $\tilde D(\theta_t)$ and the full observation of features $X \sim D(\theta_t)$ (without access to non-observed labels) to construct an estimate $\hat y$ of unseen labels $y$. Both agents are based on the causal structure among features, observations, and labels. The first step is to estimate the propensity score $\pi(x) := P(O = 1 \mid X =x)$, which reflects the probability of observation. A basic assumption of the DR Estimator is that $\pi(x) \ge \delta > 0$ for a small $\delta \in \mathbb R$, that is, every feature $x$ has a minimal probability of observation, which does not hold with thresholded decisions. To adapt to our setting, we leverage $\sigma(f_{\theta_t}(x))$ as the propensity score.

The DR Estimator agent estimates the label distribution within accepted samples using a model $\mu : \mathbb R^d \to [0, 1]$ to approximate $P(Y = 1 \mid X = x, O = 1)$. Then, the pseudo label is defined as:
\begin{align}
    y^{\text{DR}} = \mu(x) + \frac{o}{\sigma(f_{\theta_t} (x))} (y - \mu(x))
\end{align}
which then replaces the true label $y$ and is optimized using the logistic regression loss.

The DCEM \citep{chang2024biased} algorithm is based on expectation-maximization. It was designed to handle selective labels when the selection mechanism depends on a sensitive attribute, while the true label does not. The algorithm iteratively runs an expectation step (E-step) where the model at iteration $k$, $\theta(k)$ is used to define the pseudo-label:
\begin{align}
    y^{\text{DCEM}} = y \cdot o + (1 - o) \cdot \sigma(f_{\theta(k)}(x))
\end{align}
Then, the maximization step (M-step) updates the parameter to be the minimizer of:
\begin{align}
    \theta(k+1) := \argmin_\theta \E_{D}[\mathcal L(y^{\text{DCEM}}, \sigma(f_\theta(x)) ) + y^{\text{DCEM}} \mathcal L( y \cdot o , \sigma(f_\theta(x)) \cdot \sigma(f_{\theta(k)}(x))) ]
\end{align}

where $\mathcal{L}(u, v) = -u \log v - (1 - u)\log(1 - v)$ is the binary cross-entropy loss. The first term trains the predictor on the imputed pseudo-labels, while the second term acts as a causal regularizer that enforces consistency with the partial observation $Y \cdot O$. We run the M-step using gradient descent and stop after 20 iterations of the E-M cycle or when the change in one update is less than $ 10^ {- 7} $. 

\begin{figure}[t]
    \centering
    \includegraphics[width=0.85\linewidth]{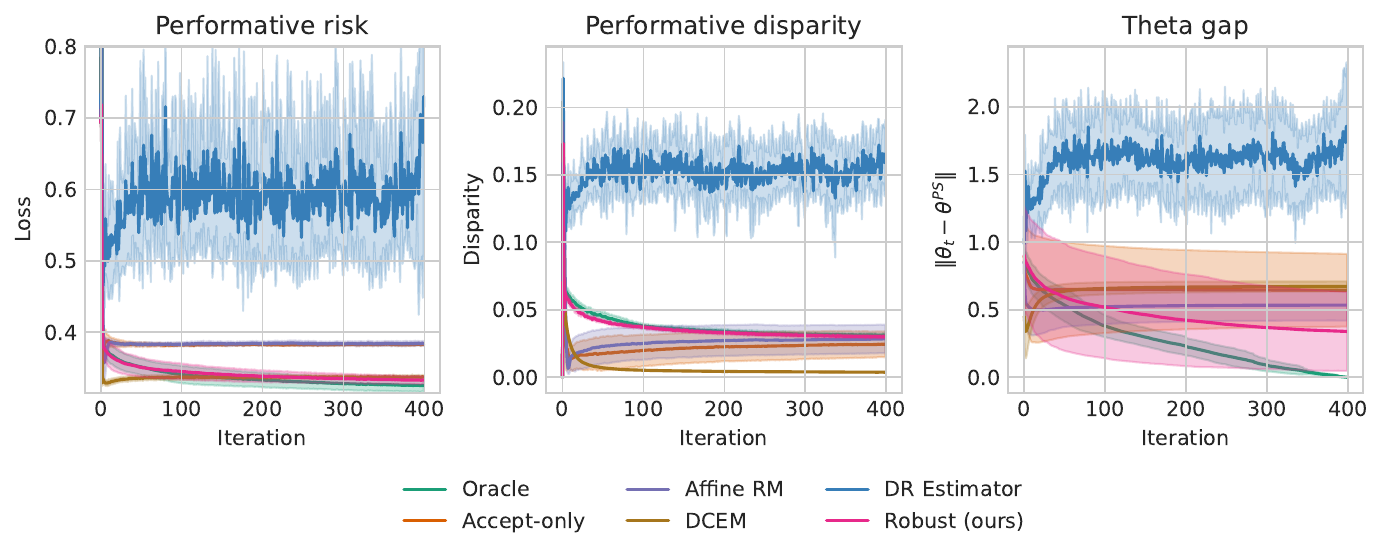}
    \caption{Loss, disparity, and distance to the stable solution presented by RRM with different objective configurations using $\epsilon = 90$. Results are averages (and std.) over 10 random seeds.}
    \label{fig:credit_90}
\end{figure}

\begin{figure}[t]
    \centering
    \includegraphics[width=0.85\linewidth]{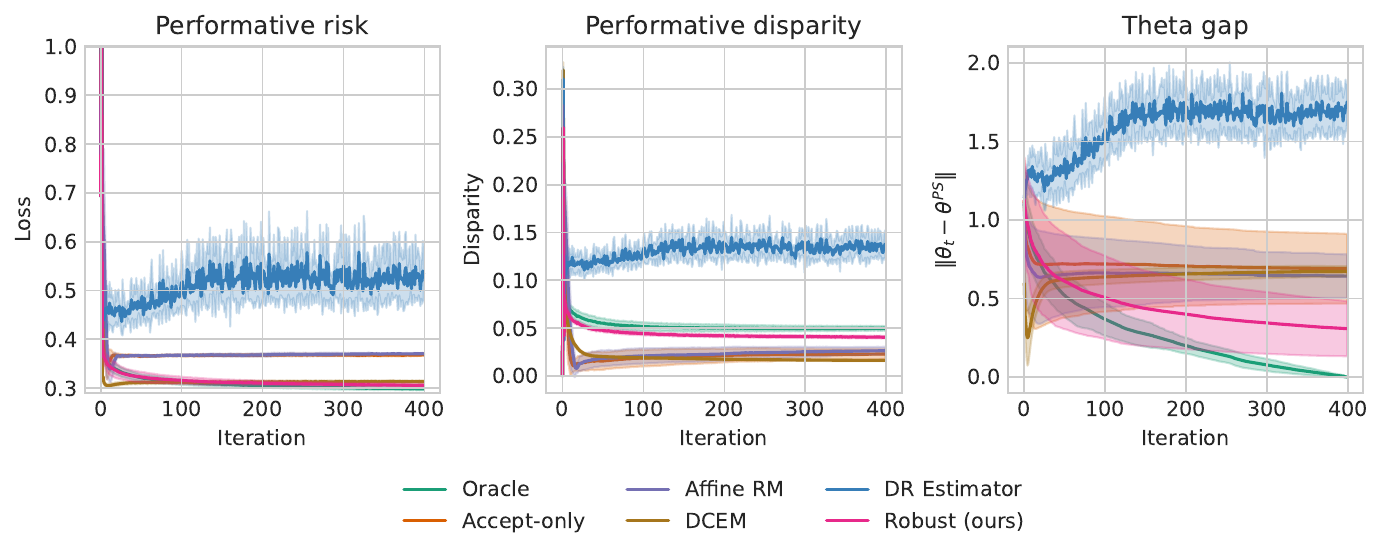}
    \caption{Loss, disparity, and distance to the stable solution for RRM with different objective configurations using $\epsilon = 150$. Results are averages (and std.) over 10 random seeds.}
    \label{fig:credit_150}
\end{figure}

\subsection{Environment Implementation}
\label{app:fico}

We implement the lending environment with a population of $n = 100{,}000$ individuals sampled from an initial distribution $D_0$ based on the FICO dataset and updated according to each deployed model $\theta$. In this environment, the feature vector was composed of the credit score, the sensitive attribute, and a constant intercept feature, $x = (c, z, 1)$. We first compute the dataset's demographic distribution, split 12\%/88\% between Black/white individuals, and set $z=1$ for Black individuals. Then, after sampling $z$, we sample a credit score $c \in \{0, \dots, 100\}$ following the empirical conditional distribution of $C \mid Z$. We calculate the empirical positive rate for each pair $(c,z)$ to estimate $\hat \alpha(x) := \widehat P(Y =1 \mid c, z)$. To handle continuous credit scores, we fit a linear interpolation to define the final label distribution. Fig.~\ref{fig:fico} depicts both $P(c|z)$ and $\alpha(x) =P(y|c,z)$.

\section{Additional Results}
\label{app:extra_results}

We report results under the same configuration as in Sec.~\ref{sec:experiments}, but with different levels of performativity, using both $\epsilon = 90$ (Fig.~\ref{fig:credit_90}) and $\epsilon = 150$ (Fig.~\ref{fig:credit_150}). In both configurations, Robust RRM achieved a performative risk similar to RRM and a small distance $\|\theta_t - \theta^{PS}\|_2$ to the stable solution. In both scenarios, Affine RM attained performative risk and disparity nearly identical to Accept-only, with only a marginally smaller distance to the stable solution. This suggests that a larger historical window does not compensate for partial label observation. In both scenarios, DCEM also matched the Oracle's performance, despite being farther from the stable solution. Future studies should examine the relationship between distance to the stable solution and distance in performative risk, extending the initial results from Appendix~\ref{app:bound_performative}.

\section{Experiments with Outcome Performativity}
\label{app:experiment_outcome}

In this section, we empirically evaluate Robust RRM when $L_\theta \neq 0$ using the limit-definition methodology in Appendix~\ref{app:results_dependent}. We adapt the lending environment based on ideas of outcome performativity \citep{dunner2022anticipating, kim2023making} where the model $\theta$ (or the final predicted score $\hat y$) has a direct effect on $y$. We use the following linear effect for $\rho \in \mathbb R^d$:
\begin{align}
    \alpha_\theta(x) = \alpha(x) + \theta^\top \rho
\end{align}
where $\alpha(x)$ is the probability of the positive label before the model effect and $\rho$ sets the strength of the effect. $\alpha_\theta(x)$ is clipped to $[0, 1]$ to represent a valid probability. If $\alpha(x)$ is $(L_x, 0)$-sensitive, then, by the Cauchy--Schwarz inequality, $\alpha_\theta(x)$ is $(L_x, \|\rho\|_2)$-sensitive. In the experiment, we set $\rho=(5,0,0)$, following the coordinates of $x = (c, z, 1)$, so only the coefficient associated with credit score affects the label distribution. This illustrates a setting in which placing greater positive weight on credit score induces greater effort to repay the loan and improve future score. We also fix $\epsilon = 50$.

We compare Robust RRM with all agents presented in Sec.~\ref{sec:experiments}. To compute label limits $u, l$, we use a modified version that only accounts for the coordinate of $\theta$ that induces performativity; that is, we replace $L_\theta \|\theta-\theta'\|_2$ by $|\rho_1| \, |[\theta]_1-[\theta']_1|$ in Eqs.~\ref{eq:lower_v2}--\ref{eq:upper_v2} ($[\theta]_1$ is the first coordinate of $\theta$).

Fig.~\ref{fig:creditoutcome} displays results for all agents. Robust RRM achieved a similar performative risk and disparity to RRM with Oracle access to labels, and it also had the smallest distance to the stable solution. Similarly to experiments without outcome performativity, using only accepted data or the DR Estimator performed the worst, with the DR Estimator agent diverging during repeated retraining. Affine RM and DCEM achieved similar distance to the stable solution; however, Affine RM displayed a higher performative risk, the same value as Accept-only. This result indicates that Robust RRM is also effective when the model affects both the feature and label distributions.

\begin{figure}[t]
    \centering
    \includegraphics[width=0.85\linewidth]{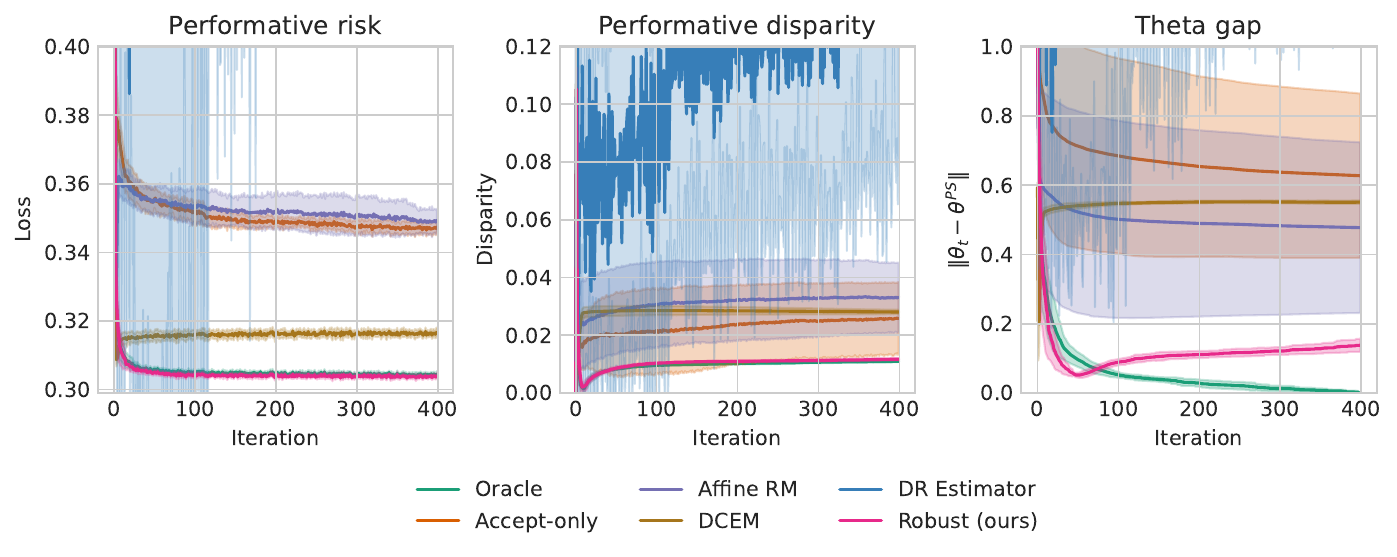}
    \caption{Loss, disparity, and distance to the stable solution in the environment with outcome performativity with $\epsilon = 50$. The y-axis was cropped to improve readability and does not display the full results for the DR Estimator agent. Results are averages (and std.) over 10 random seeds.}
    \label{fig:creditoutcome}
\end{figure}



\end{document}